\documentclass[
    sigconf = true,      
    screen = true,       
]{acmart}

\makeatletter
\def\NAT@spacechar{~}
\makeatother

\usepackage{todonotes}

\usepackage{bm}
\usepackage{upgreek}
\usepackage{mathtools}
\usepackage{xparse}
\usepackage{dsfont}
\usepackage[ruled, vlined, linesnumbered]{algorithm2e}
\usepackage[nameinlink]{cleveref}

\let\originalleft\left
\let\originalright\right
\renewcommand{\left}{\mathopen{}\mathclose\bgroup\originalleft}
\renewcommand{\right}{\aftergroup\egroup\originalright}

\ExplSyntaxOn
\DeclareExpandableDocumentCommand{\IfNoValueOrEmptyTF}{mmm}{%
    \IfNoValueTF{#1}{#2}{%
        \tl_if_empty:nTF {#1} {#2} {#3}%
    }%
}
\ExplSyntaxOff

\allowdisplaybreaks

\crefname{ineq}{inequality}{inequalities}
\creflabelformat{ineq}{#2{\upshape(#1)}#3}

\NewDocumentCommand{\probabilisticOperator}{m m}{%
    \NewDocumentCommand{#1}{m o o}{%
        #2%
        \IfNoValueOrEmptyTF{##3}{\IfNoValueTF{##3}{\left[}{[}}{\mathopen{##3[}}%
            ##1%
            \IfNoValueOrEmptyTF{##2}{}{%
                \,\IfNoValueOrEmptyTF{##3}{\IfNoValueTF{##3}{\middle|}{|}}{##3|}\,
                ##2
            }%
        \IfNoValueOrEmptyTF{##3}{\IfNoValueTF{##3}{\right]}{]}}{\mathclose{##3]}}%
    }
}
\NewDocumentCommand{\probabilisticOperatorRenew}{m m}{%
    \RenewDocumentCommand{#1}{m o o}{%
        #2%
        \IfNoValueOrEmptyTF{##3}{\IfNoValueTF{##3}{\left[}{[}}{\mathopen{##3[}}%
            ##1%
            \IfNoValueOrEmptyTF{##2}{}{%
                \,\IfNoValueOrEmptyTF{##3}{\IfNoValueTF{##3}{\middle|}{|}}{##3|}\,
                ##2
            }%
        \IfNoValueOrEmptyTF{##3}{\IfNoValueTF{##3}{\right]}{]}}{\mathclose{##3]}}%
    }
}

\NewDocumentCommand{\defineFunction}{s m m O{(} O{)}}{%
    \IfBooleanTF{#1}{%
        \NewDocumentCommand{#2}{o o}{%
            #3%
            \IfValueT{##1}{%
                \IfNoValueOrEmptyTF{##2}{\IfNoValueTF{##2}{\left#4}{#4}}{\mathopen{##2#4}}%
                ##1%
                \IfNoValueOrEmptyTF{##2}{\IfNoValueTF{##2}{\right#5}{#5}}{\mathclose{##2#5}}%
            }%
        }%
    }{%
        \NewDocumentCommand{#2}{m o}{%
            #3%
            \IfNoValueOrEmptyTF{##2}{\IfNoValueTF{##2}{\left#4}{#4}}{\mathopen{##2#4}}%
            ##1%
            \IfNoValueOrEmptyTF{##2}{\IfNoValueTF{##2}{\right#5}{#5}}{\mathclose{##2#5}}%
        }%
    }
}
\NewDocumentCommand{\redefineFunction}{m m O{(} O{)}}{%
    \RenewDocumentCommand{#1}{o o}{%
        #2%
        \IfValueT{##1}{%
            \IfNoValueOrEmptyTF{##2}{\IfNoValueTF{##2}{\left#3}{#3}}{\mathopen{##2#3}}%
            ##1%
            \IfNoValueOrEmptyTF{##2}{\IfNoValueTF{##2}{\right#4}{#4}}{\mathclose{##2#4}}%
        }
    }
}

\NewDocumentCommand{\N}{o}{%
    \mathds{N}\IfValueT{#1}{_{#1}}%
}
\NewDocumentCommand{\R}{o}{%
    \mathds{R}\IfValueT{#1}{_{#1}}%
}
\NewDocumentCommand{\eulerE}{o}{%
    \mathrm{e}\IfValueT{#1}{^{#1}}%
}

\defineFunction{\bigO}{\mathrm{O}}
\defineFunction{\bigTheta}{\Theta}
\defineFunction{\bigOmega}{\Omega}
\defineFunction{\smallO}{\mathrm{o}}
\defineFunction{\smallOmega}{\upomega}

\defineFunction{\ind}{\bm{1}}[\{][\}]
\defineFunction*{\fitness}{f}
\defineFunction*{\hammingDistance}{d_{\mathrm{H}}}
\defineFunction{\set}{}[\{][\}]
\NewDocumentCommand{\restrictedSet}{m m o}{%
\set{#1}[#3]_{#2}%
}
\NewDocumentCommand{\restrictedToN}{m o}{%
\restrictedSet{#1}{[n]}[#2]%
}

\newcommand{\dsdRLS}{d-SD RLS\xspace}
\newcommand{\flexEA}{flex-EA\xspace}
\defineFunction*{\oneMax}{\mathrm{OM}}
\newcommand{\oneMaxName}{\textsc{OneMax}\xspace}
\newcommand{\leadingOnesName}{\textsc{LeadingOnes}\xspace}
\newcommand{\shiftedOneMaxParameter}{k}
\defineFunction*{\shiftedOneMax}{\mathrm{OM}_{\shiftedOneMaxParameter}}
\newcommand*{\runTime}{T}
\newcommand{\jumpParameter}{k}
\newcommand{\stagnationWhp}{R}
\newcommand{\powerLawExponent}{\beta}
\defineFunction*{\powerLawNormalization}{N}
\defineFunction*{\onePlusOneEA}{$(1 + 1)$~EA\xspace}

\newcommand*{\numberOfFitnessLevels}{L}
\newcommand*{\improvementProbability}[1]{q_{#1}}

\newcommand*{\rateOneActive}{D}

\newcommand*{\parentRaw}{\bm{x}}
\newcommand*{\offspringRaw}{\bm{y}}
\NewDocumentCommand{\parent}{m o}{%
    \parentRaw^{(#1)}\IfNoValueF{#2}{_{#2}}
}
\NewDocumentCommand{\offspring}{m o}{%
    \offspringRaw^{(#1)}\IfNoValueF{#2}{_{#2}}
}
\NewDocumentCommand{\individual}{o}{%
    \bm{x}\IfNoValueF{#1}{_{#1}}%
}
\NewDocumentCommand{\individualOther}{o}{%
    \bm{y}\IfNoValueF{#1}{_{#1}}%
}
\NewDocumentCommand{\individualDifferent}{o}{%
    \bm{z}\IfNoValueF{#1}{_{#1}}%
}

\NewDocumentCommand{\freqs}{m o}{%
    \bm{p}^{(#1)}\IfNoValueF{#2}{_{#2}}
}

\NewDocumentCommand{\anytimeFreqs}{o}{%
    \bm{p}\IfNoValueF{#1}{_{#1}}
}

\NewDocumentCommand{\update}{o}{%
    \bm{u}\IfNoValueF{#1}{_{#1}}
}

\NewDocumentCommand{\lowerBounds}{o}{%
    \bm{\ell}\IfNoValueF{#1}{_{#1}}
}

\NewDocumentCommand{\preferredRates}{o}{%
    A\IfNoValueF{#1}{^{(#1)}}
}

\NewDocumentCommand{\counts}{m o}{%
    \bm{c}^{(#1)}\IfNoValueF{#2}{_{#2}}
}

\NewDocumentCommand{\anytimeCounts}{o}{%
    \bm{c}\IfNoValueF{#1}{_{#1}}
}

\NewDocumentCommand{\timeToNext}{o}{%
    \bm{C}\IfNoValueF{#1}{_{#1}}
}

\NewDocumentCommand{\globalCounter}{o}{%
    g\IfNoValueF{#1}{^{(#1)}}
}

\NewDocumentCommand{\globalBound}{o}{%
    G\IfNoValueF{#1}{^{(#1)}}
}

\defineFunction*{\weight}{w}

\defineFunction*{\cumulativeWeight}{W}

\NewDocumentCommand{\equilibrium}{O{\update}}{%
    q^{(#1)}
}

\probabilisticOperatorRenew{\Pr}{\mathrm{Pr}}
\probabilisticOperator{\E}{\mathrm{E}}
\probabilisticOperator{\Var}{\mathrm{Var}}

\newcommand{\jump}{\textsc{Jump}\xspace}
\newcommand{\ie}{i.\,e.\xspace}
\newcommand{\eg}{e.\,g.\xspace}
\newcommand{\whp}{w.\,h.\,p.\xspace}

\newcommand{\card}[1]{\lvert #1\rvert}
\newcommand{\sdrlsm}{\sdrlsss}
\newcommand{\fastea}{Fast (1+1)~EA\xspace}
\newcommand{\rlsonetwo}{\text{RLS}$^{1,2}$\xspace}
\newcommand{\xopt}{x_{\mathrm{opt}}\xspace}

\defineFunction{\indicator}{\mathds{1}}[\{][\}] 
\newcommand{\TwoFrequencies}{\textsc{TwoRates}\xspace}

\newcommand{\sdrlss}{SD-RLS$^{\text{r}}$\xspace}
\newcommand{\sdrlsss}{SD-RLS$^{\text{m}}$\xspace}
\newcommand{\fsdrls}{SD-FEA\xspace}
\newenvironment{proofof}[1]{\begin{proof}[Proof of~#1]}{\end{proof}}

\copyrightyear{2024}
\acmYear{2024}
\setcopyright{rightsretained}
\acmConference[GECCO '24]{Genetic and Evolutionary Computation
Conference}{July 14--18, 2024}{Melbourne, VIC, Australia}
\acmBooktitle{Genetic and Evolutionary Computation Conference (GECCO '24),
July 14--18, 2024, Melbourne, VIC, Australia}
\acmDOI{10.1145/3638529.3654076}
\acmISBN{979-8-4007-0494-9/24/07}

\acmPrice{15.00}

\begin{document}
\pagenumbering{arabic}
\title{A Flexible Evolutionary Algorithm With Dynamic Mutation Rate Archive}

\author{Martin~S. Krejca}
\orcid{0000-0002-1765-1219}
\affiliation{%
    \institution{Laboratoire d'Informatique (LIX), CNRS, École Polytechnique, Institut Polytechnique de Paris}
    \city{Palaiseau}
    \country{France}
}
\email{martin.krejca@polytechnique.edu}

\author{Carsten Witt}
\orcid{0000-0002-6105-7700}
\affiliation{%
    \institution{Technical University of Denmark}
    \city{Lyngby}
    \country{Denmark}
}
\email{cawi@dtu.dk}


\begin{abstract}
    We propose a new, flexible approach for dynamically
    maintaining successful mutation rates in evolutionary algorithms using $k$-bit flip mutations. The algorithm adds successful
    mutation rates to an archive of promising rates
    that are favored in subsequent steps. Rates expire when
    their number of unsuccessful trials has exceeded a threshold, while  rates currently
    not present in the
    archive can
    enter it in two ways: (i) via
    user-defined minimum selection probabilities for
    rates combined with a successful step or (ii) via a stagnation detection
    mechanism increasing the value for a promising rate
    after the current bit-flip neighborhood has
    been explored with high probability. For
    the minimum selection probabilities, we suggest
    different options, including heavy-tailed
    distributions.

    We conduct rigorous runtime analysis of the flexible evolutionary
    algorithm on the  \oneMaxName and \jump functions, on general unimodal functions, on minimum spanning trees, and on
    a class of
    hurdle-like functions with varying hurdle width that benefit particularly
    from the
    archive of promising mutation rates. In all cases,
    the runtime bounds are close to or even outperform
    the best known results for both
    stagnation detection and
    heavy-tailed mutations.
\end{abstract}

%
%



\maketitle

\section{Introduction}
\label{sec:introduction}

The success of evolutionary algorithms (EAs) depends crucially on their parametrization~\cite{DoerrN20}.
At their core is the careful balance of exploration and exploitation.
On the one hand, an EA should find a, potentially locally, optimal solution if it is close by.
On the other hand, it should escape local optima and keep finding better solutions within reasonable time.

To achieve this behavior, an EA needs to be flexible enough to consider solutions that are at different distances from its current solutions.
As there are many ways to make such a choice, this has led to various approaches that have been studied theoretically~\cite{DoerrN20}.
Some approaches consider a static distribution over the search radius (known as \emph{mutation rate}).
Others adjust the mutation rate dynamically during the run.
Yet other approaches allow to accept worse solutions in order to search more locally for alternative, better solutions.
And some approaches do not work with multi-sets of solutions but with distributions over solutions instead.
We discuss some prominent ideas in more detail in \Cref{sec:relatedWork}.

To the best of our knowledge, almost all of these approaches base their choice for the next mutation rate on ad hoc knowledge, not exploiting the information from previous iterations extensively.
This means that recurring structures in the search space are not necessarily well exploited.
In contrast, approaches that do attempt to incorporate knowledge from previous iterations are hard to analyze and
there are much fewer runtime results.

We propose a flexible EA (the \emph{\flexEA}, \Cref{alg:flexibleEA}), which aims to satisfy both needs.
The \flexEA maintains an archive of mutation rates that were successful in the past, called \emph{active}.
In each iteration, it picks a mutation rate based primarily on the active rates, but it still considers all other rates, based on user-defined probabilities.
Successful rates are added to the archive while unsuccessful ones are discarded after a certain amount of time, defined by the user.
If no rate is active, the \flexEA picks the successor of the last-active rate, reminiscent of stagnation detection~\cite{RajabiW23StagnationDetection}.
If the archive becomes too full without seeing any progress, it is reset.

The \flexEA allows many of its parameters to be adjusted in order to incorporate problem-specific knowledge.
However, in the absence of such knowledge, we recommend to choose the probabilities of inactive rates according to a heavy-tailed distribution~-- a common, general-purpose choice, first suggested by \citet{DoerrLMNGECCO17}, also known as a \emph{power-law}~-- and to choose the times for discarding active rates in the same manner as done in stagnation detection~\cite{RajabiW23StagnationDetection}.
This results in an essentially parameter-less algorithm. 

We show the efficiency of the \flexEA by analyzing it on a diverse set of problems, namely, on unimodal and on jump functions, on the minimum spanning tree problem, and on hurdle-like functions with varying hurdle width.
For all of these problems, the performance of the \flexEA is close to or even better than the best known results for stagnation detection and heavy-tailed mutation.

For unimodal functions (\Cref{sec:unimodalFunctions}), the progress of the \flexEA is mainly determined by its probability to improve solutions locally (\Cref{thm:runTimeOnUnimodal}).
For \oneMaxName and \leadingOnesName of problem size~$n$ and with the recommended parametrization, this results in the common bounds of $\bigO{n \log(n)}$ and $\bigO{n^2}$, respectively.

For the $\jump$ benchmark of problem size~$n$ and gap size~$\jumpParameter$ (\Cref{sec:jump}), the recommended parametrization results for $\jumpParameter = \smallO{n/\log n}$ in a bound of $\big(2 + \smallO{1}\big)\binom{n}{\jumpParameter}$ (\Cref{cor:runTimeOnJumpPowerLaw}), which is optimal for unbiased algorithms with unary mutation, up to a factor of $2 + \smallO{1}$~\cite{DoerrR23}.
For many other values of~$\jumpParameter$, the result is a product of~$\binom{n}{\jumpParameter}$ and a factor that depends at most quadratically on~$\jumpParameter$.

For the combinatorial problem of finding a minimum spanning tree on a graph with~$m$ edges (\Cref{sec:mst}), we rely on the flexibility 
to incorporate problem-specific knowledge, and we adjust the parameters such that the algorithm mainly makes progress by flipping one or two edges into or out of solutions.
With this choice, the \flexEA has an expected runtime of $\bigl(1 + \smallO{1}\bigr)\bigl(m^2 \log(m)\bigr)$ (\Cref{theo:mst}), which matches the best known bound~\cite{RajabiWittGECCO21}.

Last, to showcase the benefit of an archive, we introduce a problem that features many local optima separated by gaps of alternating widths (\Cref{sec:twoRates}).
Once the \flexEA determines the necessary rates, it maintains them successfully until the global optimum is found.
It does so in an expected time that is faster by a poly-logarithmic factor than heavy-tailed mutation and faster by a super-polynomial factor than stagnation detection (\Cref{theo:twofreqcs}).
Both, heavy-tailed mutation and stagnation detection are slowed down by requiring to find the correct mutation rate repeatedly, although heavy-tailed mutation does this far more efficiently than stagnation detection.

Overall, our results show that the \flexEA has a great performance with the recommended parameter setting while effectively not requiring any parameter choices.
In settings with repeating local structures, this parametrization even outperforms the operators it is based on.
In the case of problem-specific knowledge, the \flexEA can achieve results that match the best known bounds.
For all of these settings, the state space of the algorithm remains simple enough that a mathematical analysis is still tractable.

\section{Related Work}
\label{sec:relatedWork}

There exists a plethora of different operators for choosing good mutation rates~\cite{DoerrN20}.
We provide an overview on established concepts that are more closely related to the \flexEA.
In particular, we only concern unary operators.
Furthermore, we focus on single-trajectory algorithms that only maintain a single search point at a time and update it iteratively.
By using populations and diversity-maintaining operators, one opens up for a new level of complexity in the design and analysis of evolutionary algorithms.
In addition, this excludes estimation-of-distribution algorithms (EDAs) ~\cite{KrejcaW20EDABookChapter}, which maintain a probability distribution over the search space instead of a single search point.


\paragraph{Non-elitism.}
A non-elitist algorithm is one in which worse solutions can be selected over better solutions.
Although this approach does not directly change the mutation rate, it can affect the success of the current rate by creating solutions in a distance that is more favorable \cite{DangEL21AAAI21,DoerrHRWGECCO23}.
This requires a careful balance of the selection probability of solutions.
In comparison, the \flexEA is an elitist algorithm, \ie, it aims at finding mutation rates that are good for the currently best solution without modifying it.

\paragraph{Static rate distributions.}
A large class of EAs, most prominently the \onePlusOneEA, use a static distribution for selecting mutation rates.
Since this distribution is fixed, it needs to accommodate all situations that may arise during optimization.
The \onePlusOneEA uses a single rate that creates solutions with high probability locally around the current search point.
This leads to long waiting times for escaping local optima~\cite{JansenW99}.

Artificial immune systems (AISes) use operators that change solutions more drastically \cite{CorusOY21tec}.
However, since this makes it harder to find local improvements, AISes often use additional operators that influence how large the changes to a solution actually are~\cite{ZargesAISBookChapter}.

Arguably the best of both worlds is achieved by choosing the mutation rate according to a heavy-tailed distribution, as introduced by \citet{DoerrLMNGECCO17} in their seminal paper.
This approach, known as fast mutation, finds local improvements quickly and also has a good probability to use larger mutation rates to escape local optima.
Since its inception, this approach has seen successful applications in a variety of theoretical works~\cite{FriedrichQW18,FriedrichGQW18,WuQT18,AntipovD20ppsn,QuinzanGWF21,DoerrZ23MultiObjective,CorusOY21tec,AntipovBD22,DangELQ22,DoerrGI22,DoerrR23,DoerrQ23tec}.
A drawback of fast mutation is that it never adjusts its distribution.
Thus, to find improving solutions at larger distances, it exclusively relies on the decent probability of choosing a sufficiently large rate.
Depending on the distance, this can result in long waiting times.

\paragraph{Dynamic rate distributions.}
Some algorithms change the distribution for choosing the mutation rate during the optimization process~\cite{DDParameterControl2020}.
While this technically allows for complex distributions, to the best of our knowledge, most algorithms that fall into this category distribute the probability mass around a single rate.

A common approach is the $1/5$ success rule (\eg, \cite{SchumerSteiglitz1968, Rechenberg1973, DoerrDoerrGECCO15}), which increases the current rate if more than one fifth of the last tries with the current rate are successful, and it decreases it otherwise.
While this approach automatically adjusts the rate to one for finding an improving solution, it forgets the previously used rate and, instead, aims to rediscover it if it becomes useful again later.

A more deterministic approach to a similar idea is stagnation detection \cite{RajabiWittAlgo22,RajabiW23StagnationDetection}.
This operator starts with the smallest rate and picks it repeatedly until the algorithm is certain with high probability that it cannot find an improving solution with the current rate.
It then increases the rate and repeats the process.
This approach has the benefit that it can escape local optima very well.
However, similar to the one-fifth success rule, the algorithm focuses on a single rate, meaning that it needs to rediscover previously good rates.
This problem was addressed partially by adding a radius memory to the algorithm \cite{RajabiWittGECCO21}, which saves the last successful rate in order to re-use it.
However, if more than a single previous rate are useful, this approach falls back to the problem of the initial algorithm.

Another, very recent approach~\cite{DoerrR23} proposes to combine stagnation detection with fast mutation in the following way:
The algorithm follows classic stagnation detection but does not always create a solution with the current mutation rate; it also has a chance of picking a mutation rate from a heavy-tailed distribution around the current rate.
This proves very effective in escaping certain local optima where there are many improving solutions at a further distance~\cite{DoerrR23}.
The resulting runtime is better than that of classic stagnation detection and thus also than that of the \flexEA.
However, we believe that this variant still has problems of re-using previously good rates, as it does not save them.

\paragraph{Learning-based algorithms.}
\citet{DoerrDYPPSN16} propose an algorithm that chooses its mutation rate according to a distribution that is adjusted over time.
This adjustment follows the idea of reinforcement learning and incorporates information from many past iterations.
A related concept are \emph{hyper-heuristics}, e.g.,~\cite{LissovoiOW20HyperHeuristics,LissovoiOW23HHvsMetropolis}, which are algorithms that aim at selecting in each iteration an algorithm that is best suited for the current sate.
This choice is based on the previously observed performance of the algorithms to select from.

The \flexEA is similar to these approaches, with the main difference being that it updates its distribution over the mutation rate in a simpler way.
This simplification especially allows a rigorous analysis of the \flexEA on a wide variety of problems while still showcasing a strong performance.

\section{Preliminaries}
\label{sec:preliminaries}

The natural numbers~$\N$ include~$0$.
For $a, b \in \R$, let $[a .. b] = [a, b] \cap \N$, and let $[a] = [1 .. a]$. By $\log(\cdot)$ we denote
the binary logarithm.

We consider pseudo-Boolean optimization, that is, for a given $n \in \N_{\geq 1}$, the optimization of functions $\fitness\colon \{0, 1\}^n \to \R$.
We call~$\fitness$ a \emph{fitness function}, and we assume that~$n$ is given implicitly.
When using big-O notation, we assume asymptotics in this implicit~$n$.
Furthermore, we call each $\individual \in \{0, 1\}^n$ an \emph{individual}, and~$\fitness[\individual]$ its \emph{fitness}.
For each $i \in \N$, let~$\individual[i]$ denote the value of~$\individual$ at position~$i$.
Last, let~$|\individual|_1$ be the number of~$1$s of~$\individual$, and~$|\individual|_0$ its number of~$0$s.

For each $\individual, \individualOther \in \{0, 1\}^n$, let $\hammingDistance[\individual, \individualOther] = |\set{i \in \N \mid \individual[i] \neq \individualOther[i]}|$ denote the Hamming distance of~$\individual$ and~$\individualOther$.
For each $r \in [n]$, we define the \emph{$r$-bit flip neighborhood of~$\individual$} to be the set $\set{\individualOther \in \{0, 1\}^n \mid \hammingDistance[\individual, \individualOther] = r}$.

Given a fitness function~$\fitness$ and an algorithm~$A$ optimizing~$\fitness$, we define the \emph{runtime} of~$A$ on~$\fitness$ as the (random) number of evaluations of~$\fitness$ until the first time that an optimum of~$\fitness$ is evaluated.
To this end, we assume that each individual that~$A$ creates is evaluated exactly once.
For the \flexEA in \Cref{sec:flexEA}, the runtime is essentially the number of iterations until an optimum is created for the first time (plus one, due to the initialization).

\section{The Flexible Evolutionary Algorithm}
\label{sec:flexEA}

The \emph{flexible evolutionary algorithm} (\flexEA, \Cref{alg:flexibleEA}) is an elitist $(1 + 1)$-type evolutionary algorithm
choosing its mutation rate each iteration according to a probability distribution given by a vector (the \emph{frequency vector}).
The \flexEA updates its frequency vector in each iteration, aiming to give promising rates a high probability.
This update is influenced by different factors.
In the following, we first discuss the high-level idea of the \flexEA and its different parts.
Afterward, in \Cref{sec:flexEA:recommendation}, we propose good default choices for the various parameters of the algorithm.
The result is a basically parameter-less algorithm that can be thought of as a superposition of heavy-tailed mutation~\cite{DoerrLMNGECCO17} and stagnation detection~\cite{RajabiW23StagnationDetection}.

\paragraph{Update of the frequency vector.}
The update of the 
vector $\anytimeFreqs \in [0, 1]^n$ aims to give (mutation) rates a high probability if choosing them in the past
succeeded in improving the fitness of the current solution.
If a rate repeatedly fails to be successful, the probability of choosing it (called its \emph{frequency}) is reduced.
To this end, the \flexEA groups all rates into those that receive a high frequency (called \emph{active}) and those that receive a low frequency (called \emph{inactive}).

The active frequencies follow a uniform distribution.
Each active frequency has a counter of how often it failed to succeed.
If this counter exceeds a user-defined limit, the respective frequency becomes inactive.
If this results in no frequencies being active and the frequency at index $r \in [n]$ was the last active one, then the frequency at index $r + 1$ is set to active (identifying $n+1$ with $1$).

The update process is superimposed by user-defined lower bounds per frequency.
If a frequency would be set to a value less than its lower bound, it takes on its lower bound instead.

In addition, the \flexEA keeps a global counter that counts how many times in a row no success was achieved with any mutation rate.
If this counter reaches a specific limit, all frequencies become inactive, and the frequency at position~$1$ becomes active.
This effectively resets the frequency vector.

\paragraph{The \flexEA in more detail.}
\Cref{alg:flexibleEA} shows the detailed \flexEA.
The frequency vector is denoted by $\anytimeFreqs \in [0, 1]^n$, and the set of active rates (called the \emph{archive})
by $\preferredRates \subseteq [n]$.
The lower bounds for each frequency are provided by the user to the \flexEA via the \emph{lower-bound vector} $\lowerBounds \in [0, 1]^n$.
We assume that $\sum_{i \in [n]} \lowerBounds[i] \leq 1$.

\begin{algorithm}[tb]
    \caption{\label{alg:flexibleEA}
        The flexible evolutionary algorithm (\flexEA) with given lower-bound vector~$\lowerBounds$ and count bound vector~$\timeToNext$, maximizing a pseudo-Boolean function $f\colon \{0, 1\}^n \to \R$, stopping at a user-defined termination criterion.
    }
    \newcommand*{\redist}{\mathrm{toRedist}}
    $t \gets 0$\;
    $\globalCounter[t] \gets 0$\;
    $\preferredRates[t] \gets \{1\}$\;
    \lFor{$i \in [n]$}{%
    $\counts{t}[i] \gets 0$%
    }
    $\parent{t} \gets$ solution from $\{0, 1\}^n$ chosen uniformly at random\;
    \While{\emph{termination criterion not met}}{%
        $\freqs{t} \gets$ distribute probability mass over~$\preferredRates[t]$, respecting~$\lowerBounds$ (\Cref{alg:distributeMass})\;
        $m \gets \min \preferredRates[t]$\;
        $\globalBound[t] \gets \frac{\timeToNext[m]}{\freqs{t}[m]}$\;
        choose $r \in [n]$ according to~$\freqs{t}$\;
        $\offspring{t} \gets$ copy~$\parent{t}$ and flip~$r$ bits uniformly at random in this copy\;
        \eIf{$\fitness[\offspring{t}][\big] > \fitness[\parent{t}][\big]$}{%
            $\parent{t + 1} \gets \offspring{t}$\;
            $\preferredRates[t + 1] \gets \preferredRates[t] \cup \{r\}$\;
            $\globalCounter[t + 1] \gets 0$\;
            $\counts{t + 1}[r] \gets 0$\;
        }{%
            \lIf{$\fitness[\offspring{t}][\big] = \fitness[\parent{t}][\big]$}{%
                $\parent{t + 1} \gets \offspring{t}$%
            }
            $\globalCounter[t + 1] \gets \globalCounter[t] + 1$\;
            $\counts{t + 1}[r] \gets \counts{t}[r] + 1$\;
            \If{$\globalCounter[t + 1] \geq \globalBound[t]$}{%
                $\globalCounter[t + 1] \gets 0$\;
                $\preferredRates[t + 1] \gets \{1\}$\;
                $\counts{t + 1}[1] \gets 0$\;
            }
            \ElseIf{$\counts{t + 1}[r] \geq \timeToNext[r]$}{%
                $\preferredRates[t + 1] \gets \preferredRates[t] \smallsetminus \{r\}$\;
                \If{$\preferredRates[t + 1] = \emptyset$}{%
                    $r' \gets$ \leIf{$r + 1 \leq n$}{$r + 1$}{$1$}
                    $\preferredRates[t + 1] \gets \preferredRates[t + 1] \cup \{r'\}$\;
                    $\counts{t + 1}[r'] \gets 0$\;
                }
            }
        }
        $t \gets t + 1$\;
    }
\end{algorithm}

Furthermore, each rate $i \in [n]$ has a counter $\anytimeCounts[i] \in \N$ for counting its number of failed attempts to improve the current solution.
If~$\anytimeCounts[i]$ reaches a user-defined limit $\timeToNext[i] \in \R_{\geq 0}$, then frequency~$i$ becomes inactive.
These limits are provided by the user as the \emph{count bound vector} $\timeToNext \in \R_{\geq 0}^n$.

The global counter is denoted by $\globalCounter \in \N$.
Its limit that, if reached, resets the archive of the \flexEA is denoted by $\globalBound \in \R_{\geq 0}$.
The value~$\globalBound$ is chosen as the count bound of the smallest active rate $m \in [n]$ divided by its current frequency.
This resembles the normal bound~$\timeToNext[m]$ but accounts for rate~$m$ being only chosen once in~$\anytimeFreqs[m]^{-1}$ iterations in expectation.

The update of the frequency vector is a bit involved, as frequencies may be capped by their respective lower bound, thus disallowing to distribute the probability mass uniformly among all active frequencies.
Note that the \flexEA aims to give each active frequency the same probability.
The total probability to distribute among these $|\preferredRates| \eqqcolon \varphi$ frequencies is $1 - \sum_{i \in [n] \smallsetminus \preferredRates} \lowerBounds[i] \eqqcolon M$, as all inactive frequencies are at their respective lower bound.
Hence, each active frequency should get a probability of~$\frac{M}{\varphi}$.
In order to account for the lower bounds of the active frequencies, the active frequencies are adjusted sequentially, starting with the largest lower bound, going to the smallest.
If the lower bound of the first frequency is less than the intended value~$\frac{M}{\varphi}$, then all active frequencies get this value.
Otherwise, the first frequency is set to its lower bound, and the remaining probability mass is adjusted.
This process is repeated.

\Cref{alg:distributeMass} shows this approach algorithmically.
We note that the algorithm is presented in a way that aims to easily understand the logic of the update.
It is not optimized for efficiency.
A more efficient implementation would make use of a data structure for the archive that quickly allows to add and delete indices while keeping them ordered with respect to their lower bound.
This can be achieved, for example, with a balanced search tree.

\begin{algorithm}[t]
    \caption{\label{alg:distributeMass}
        The algorithm that, given an index set $\preferredRates \subseteq [n]$ and a lower-bound vector~$\lowerBounds$ of size~$n$, returns a frequency vector~$\anytimeFreqs$ of size~$n$ where all possible probability mass is distributed as evenly as possible over the frequencies with indices in~$\preferredRates$, respecting~$\lowerBounds$.
    }
    $M \gets 1 - \sum_{i \in [n] \smallsetminus \preferredRates} \lowerBounds[i]$\;
    \lFor{$i \in [n] \smallsetminus \preferredRates$}{%
        $\anytimeFreqs[i] \gets \lowerBounds[i]$%
    }
    $\varphi \gets |\preferredRates|$\;
    $(S_i)_{i \in [\varphi]} \gets \preferredRates$ sorted in descending order by~$\lowerBounds$\;
    \For{$i \in [\varphi]$}{%
        \eIf{$\lowerBounds[S_i] \leq \frac{M}{\varphi - i + 1}$}{%
            \lFor{$j \in [i .. \varphi]$}{%
                $\anytimeFreqs[S_j] \gets \frac{M}{\varphi - i + 1}$%
            }
            \Return{$\anytimeFreqs$}\;
        }{%
            $\anytimeFreqs[S_i] \gets \lowerBounds[S_i]$\;
            $M \gets M - \lowerBounds[S_i]$\;
        }
    }
    \Return{$\anytimeFreqs$}\;
\end{algorithm}

\subsection{Recommended Parameter Choices}
\label{sec:flexEA:recommendation}

The \flexEA requires the user to provide~$2n$ parameters, which can be daunting.
Hence, we propose a recommendation that works well in many cases, as we prove in most of the following sections.
For the lower-bounds vector~$\lowerBounds$, we recommend to choose a heavy-tailed distribution, inspired by the fast-mutation operator~\cite{DoerrLMNGECCO17}.
For the count bound vector~$\timeToNext$, we recommend to choose the same values as for stagnation detection~\cite{RajabiW23StagnationDetection}.
We explain both recommendations in more detail below and then briefly discuss the resulting algorithm.

\paragraph{Heavy-tailed lower bounds.}
Let $\powerLawExponent \in (1, 2)$ (called the \emph{power-law exponent}), and let $\powerLawNormalization\colon (1, 2) \to \R_{\geq 1}$ denote a normalization with $\powerLawNormalization\colon \alpha \mapsto \sum_{i \in [n]} i^{-\alpha}$.
We recommend for all $i \in [n]$ to choose
\begin{align*}
    \lowerBounds[i] = \frac{1}{2 \powerLawNormalization[\beta]} i^{-\beta} .
\end{align*}

Note that $\sum_{i \in [n]} \lowerBounds[i] = \frac{1}{2}$.
Hence, the \flexEA still has a probability mass of~$\frac{1}{2}$ remaining to distribute among the active frequencies.

This recommendation follows the ideas of \emph{fast mutation}, introduced in  \cite{DoerrLMNGECCO17}.
This mutation operator first chooses a rate $r \in [n]$ according to a heavy-tailed distribution and performs standard bit mutation afterward.
That is, it flips each bit of a given individual independently with probability~$\frac{r}{n}$.
\citet{DoerrLMNGECCO17} name the variant of the \onePlusOneEA that uses fast mutation the \fastea.

\citet{DoerrLMNGECCO17} choose the heavy-tailed distribution because even large rates have a reasonable probability of at least~$\frac{1}{n^2}$ of being chosen.
The authors show that the \fastea has a super-exponential runtime speed-up in the parameter~$\jumpParameter$ compared to the classic \onePlusOneEA when optimizing $\jump_\jumpParameter$ and has seen great success in other settings (see also \Cref{sec:relatedWork}).
Hence, we propose to also use it for the \flexEA.

\paragraph{Stagnation detection bounds.}
Let $\stagnationWhp \in \R_{> 1}$.
We recommend 
\begin{align*}
    \timeToNext[i] = \binom{n}{i} \ln(\stagnationWhp)
\end{align*}
for all $i \in [n]$ and
call this the \emph{standard-SD choice}.

This recommendation follows the ideas of \emph{stagnation detection} (SD), introduced by \citet{RajabiW23StagnationDetection}.
In their paper, the authors introduce the algorithm \emph{\sdrlss}, which is also a $(1 + 1)$-type elitist algorithm that performs the same mutation as the \flexEA but chooses its mutation rate differently in the following deterministic manner.
The \sdrlss starts with rate~$1$ and counts how many times it used this rate.
It then uses rate~$1$ for at most~$\timeToNext[1]$ iterations.
If it is successful in finding a strictly improving solution during this time, the algorithm resets its counter.
If rate~$1$ is not successful, the \sdrlss chooses rate~$2$ from now on, for a total of at most~$\timeToNext[2]$ iterations.
If it is successful during this time, it resets its counter and reverts back to rate~$1$.
Otherwise, it picks the next-higher rate and adjusts the iteration limit according to the formula above.
We note that the \sdrlss only tries out rates up to~$\frac{n}{2}$, but the formula for the bounds generalizes to all values in~$[n]$.

\citet{RajabiW23StagnationDetection} choose the bounds above because this guarantees that if rate $r \in [n]$ is currently used and there is at least one solution in Hamming distance~$r$ of the currently best solution, the \sdrlss find this improvement with probability at least $1 - \stagnationWhp^{-1}$, as the probability of not doing so is at most $\big(1 - \binom{n}{r}^{-1}\big)^{\timeToNext[r]} \leq \stagnationWhp^{-1}$.
Hence, the value~$\stagnationWhp$ determines how unlikely it is to fail in such a case.
Since the \flexEA is in some aspects similar to the \sdrlss, we recommend to copy its bounds.

\paragraph{Discussion}
When considering the \flexEA with heavy-tailed lower bounds and with the standard-SD choice, it has two parameter values left to choose, namely, the power-law exponent $\powerLawExponent \in (1, 2)$ as well as the parameter $\stagnationWhp \in \R_{> 1}$.
However, the exact choice of these parameters is typically not very relevant, as we show in this article.
\citet{DoerrLMNGECCO17} recommend to choose $\beta = \frac{3}{2}$.
Since the \flexEA can act similarly to the algorithm discussed by \citet{DoerrLMNGECCO17}, we recommend this choice too.
For the parameter~$\stagnationWhp$, \citet{RajabiW23StagnationDetection} show that a choice of $\stagnationWhp \geq n^{4+\epsilon}$ for a small constant
$\epsilon>0$, with~$\stagnationWhp$ being a polynomial, is usually sufficient.
We show in our results below that this is also the case for the \flexEA.

Since we provide recommendations for the only two parameter choices left, we consider the \flexEA to be an essentially parameter-less algorithm.
However, one needs to keep in mind that the \flexEA allows to choose its parameters~$\lowerBounds$ and~$\timeToNext$ more flexibly, which can be a good choice if given extra information (see also \Cref{sec:mst}).

\section{Runtime on Unimodal Functions}
\label{sec:unimodalFunctions}

We analyze the \flexEA on unimodal pseudo-Boolean functions, that is, functions $\fitness\colon \{0, 1\}^n \to \R$ with a unique global optimum $\individualDifferent^* \in \{0, 1\}^n$ such that, for all $\individual \in \{0, 1\}^n \smallsetminus \{\individualDifferent^*\}$, there is a $\offspringRaw \in \{0, 1\}^n$ with $\hammingDistance[\individual, \offspringRaw] = 1$ such that $\fitness[\individual] < \fitness[\offspringRaw]$.
For each unimodal function~$\fitness$, we call the cardinality of the range of~$\fitness$ the number of \emph{fitness levels of~$\fitness$}, and each set of all individuals of the same fitness a \emph{fitness level}.
Each individual (besides the global optimum) has at least one neighbor in a fitness level of better fitness.

The \flexEA optimizes unimodal functions efficiently in expectation if its parameters are chosen well.
To this end, we rely exclusively on the lower bound of rate~$1$.

\begin{theorem}
    \label{thm:runTimeOnUnimodal}
    Let~$\fitness$ be a unimodal fitness function with~$\numberOfFitnessLevels$ fitness levels.
    For all $j \in [L - 1]$, let~$\improvementProbability{j}$ denote a lower bound on the probability that, given an individual~$\individualDifferent$ in fitness level~$j$, flipping a single bit in~$\individualDifferent$ creates an offspring with strictly better fitness.
    Furthermore, consider the \flexEA with $\lowerBounds[1] > 0$.

    Then the expected runtime of the \flexEA on~$\fitness$ is at most
    \begin{align*}
        1 + \lowerBounds[1]^{-1} \sum\nolimits_{j \in [\numberOfFitnessLevels - 1]} \improvementProbability{j}^{-1} .
    \end{align*}
    Especially, it is at most $1 + n (\numberOfFitnessLevels - 1) \lowerBounds[1]^{-1}$.
\end{theorem}

\begin{proof}
    We start with the first claim.
    The plus~$1$ is for
    initialization.
    For the remaining bound, we apply the fitness level method (\Cref{thm:fitnessLevelMethod}).
    Let $j \in [L - 1]$ be the fitness level of the current solution of the \flexEA.
    In order to leave level~$j$, it is sufficient to choose rate~$1$ and then create an offspring with strictly better fitness.
    The probability of the former is~$\lowerBounds[1]$, and the one of the latter is at least~$\improvementProbability{j}$.
    Since these events are independent, the \flexEA leaves level~$j$ with a probability of at least~$\lowerBounds[1] \improvementProbability{j}$.
    This concludes this case.

    For the second claim, note that since~$f$ is unimodal, there is always at least one improving solution that only requires a single bit to be flipped.
    Hence, it holds that $\improvementProbability{j} \geq \frac{1}{n}$.
    Applying the first claim concludes the proof.
\end{proof}

\Cref{thm:runTimeOnUnimodal} leads to useful bounds on various functions.
For example, the coarser bound at the end results already in the common expected runtime of $\bigO{n^2}$ for the \leadingOnesName benchmark~\cite{Rudolph1997} if we assume a constant lower bound for rate~$1$, as is the case with heavy-tailed lower bounds.
For functions where the probability of leaving a fitness level depends on the current state, we need to estimate these transition probabilities more carefully.
Doing so leads to an expected runtime bound of $\bigO{n \log n}$ for the \oneMaxName benchmark, as we show below (\Cref{cor:runTimeOnShiftedOneMax}).
Since we require bounds for similar functions in the following sections, we introduce a slight generalization of \oneMaxName and analyze it instead.

For all $\shiftedOneMaxParameter \in [0 .. n]$, let \emph{trimmed \oneMaxName} ($\shiftedOneMax$) be defined as
\begin{align*}
    \individual \mapsto
    \begin{cases}
        \shiftedOneMaxParameter + \oneMax[\individual] & \textrm{if $|\individual|_1 \leq n - \shiftedOneMaxParameter$,}\\
        n - \oneMax[\individual] & \textrm{else.}
    \end{cases}
\end{align*}
Note that $\shiftedOneMax$ is a unimodal function with $n + 1$ fitness levels.
Its global optimum is achieved for all individuals with exactly $n - \shiftedOneMaxParameter$ $1$s.
The case $\shiftedOneMaxParameter = 0$ results in \oneMaxName.
We get the following bound.

\begin{corollary}
    \label{cor:runTimeOnShiftedOneMax}
    Let $\shiftedOneMaxParameter \in [0 .. n]$.
    Consider the \flexEA with $\lowerBounds[1] > 0$ optimizing $\shiftedOneMax$.
    Then the expected runtime is $\bigO{\lowerBounds[1]^{-1} n \log(n)}[\big]$.
\end{corollary}

\begin{proof}
    We apply \Cref{thm:runTimeOnUnimodal} and use its notation.
    For all $j \in [n + 1]$, denote the fitness level that represents an $\shiftedOneMax$-value of $j - 1$.
    We say that individuals with at most $n - \shiftedOneMaxParameter$ $1$s are on the \emph{left branch} of the function, and the other ones on the \emph{right branch}.

    Let $j \in [n]$.
    We bound~$\improvementProbability{j}$ with respect to whether the current individual is on the left or the right branch.
    If it is on the left branch, it has $j - 1 - \shiftedOneMaxParameter$ $1$s.
    If mutation flips any of the $n - j + 1 + \shiftedOneMaxParameter$ $0$s, then the result strictly improves the fitness.
    Hence, $\improvementProbability{j} = \frac{n - j + 1 + \shiftedOneMaxParameter}{n}$.

    If the current individual is on the right branch, it has $n - j + 1$ $1$s.
    If mutation flips any of these $1$s, then the result strictly improves the fitness.
    Hence, $\improvementProbability{j} = \frac{n - j + 1}{n}$.

    Combining both cases, by \Cref{thm:runTimeOnUnimodal} and disregarding the plus~$1$, we obtain an upper bound on the expected runtime of
    \begin{align*}
        &\lowerBounds[1]^{-1} \sum_{j \in [n]} \left(\frac{n}{n - j + 1 + \shiftedOneMaxParameter} \cdot \indicator{j \geq \shiftedOneMaxParameter + 1} + \frac{n}{n - j + 1} \cdot \indicator{j < \shiftedOneMaxParameter + 1}\right)\\
        &\leq \lowerBounds[1]^{-1} \cdot 2 n \sum\nolimits_{j \in [n]} j^{-1}
        = \bigO{\lowerBounds[1]^{-1} n \log(n)}[\big] .\qedhere
    \end{align*}
\end{proof}

\section{Runtime on \jump Functions}
\label{sec:jump}

We consider the well known \jump benchmark, which, for all $\jumpParameter \in [n]$ is defined as
\begin{align*}
    \individual \mapsto
    \begin{cases}
        \jumpParameter + \oneMax[\individual] & \textrm{if $|\individual|_1 \in [n - \jumpParameter] \cup \{n\}$,}\\
        n - \oneMax[\individual] & \textrm{else.}
    \end{cases}
\end{align*}
This function is identical to~$\shiftedOneMax$ except that its unique global optimum is at the all-$1$s bit string.
Consequently, the individuals with exactly $n - k$ $1$s are all local optima.
The set of all local optima is often called the \emph{plateau} of the function.
For elitist algorithms, in order to leave the plateau, exactly all~$\jumpParameter$ $0$s of an individual need to be changed into a~$1$, which requires, in expectation, at least~$\binom{n}{\jumpParameter}$ tries for unbiased algorithms with unary mutation~\cite{DoerrR23}.

With \Cref{cor:runTimeOnJumpPowerLaw}, we show that the \flexEA with the recommended parameter choices optimizes \jump in this time.
Beforehand, we prove a runtime bound for a more general parametrization.

\begin{theorem}
    \label{thm:runTimeOnJump}
    Let $\stagnationWhp \in \R_{> 1}$ and $\jumpParameter \in [2 .. n]$.
    Consider the \flexEA with the standard-SD choice~$\timeToNext$ with parameter~$\stagnationWhp$ and with lower bounds~$\lowerBounds$ such that $\lowerBounds[1], \lowerBounds[\jumpParameter] > 0$ and $1 - \sum_{i \in [n] \smallsetminus \{\jumpParameter\}} \lowerBounds[i] = \bigTheta{1}$.
    Furthermore, consider that the \flexEA optimizes $\jump_\jumpParameter$.
    Let $L = (1 - \sum\nolimits_{i \in [n] \smallsetminus \{\jumpParameter\}} \lowerBounds[i])^{-1}$.
    Then the expected runtime is
    \begin{align*}
        &\bigO{n \lowerBounds[1]^{-1} \log(n \stagnationWhp)}
            + \min \Bigg\{
                \binom{n}{\jumpParameter} \cdot \lowerBounds[k]^{-1},\\
                &\qquad\binom{n}{\jumpParameter} \cdot \left(L
                    + \bigO{\frac{\jumpParameter}{n - 2\jumpParameter + 3} \log(\stagnationWhp)}
                    + \lowerBounds[\jumpParameter]^{-1} \stagnationWhp^{-1/L}
                \right)\\
                &\qquad+ 2^n \cdot \indicator{k \geq \frac{n}{2}} \ln(\stagnationWhp)
            \Bigg\} .
    \end{align*}
\end{theorem}

\begin{proof}
    Let~$\runTime$ denote the runtime of the algorithm on~$\jump_\jumpParameter$.
    We split~$\runTime$ into the following two phases:
    Phase~$1$ considers the time until the current solution of the algorithm has a fitness of at least~$n$, that is, the current solution is on the plateau or optimal.
    Phase~$2$ starts in the iteration that phase~$1$ stops and considers the time until the current solution is optimal.
    Let~$\runTime_1$ and~$\runTime_2$ denote the runtimes of the respective phases.
    By the linearity of expectation and by $\runTime = \runTime_1 + \runTime_2$, it holds that $\E{\runTime} = \E{\runTime_1} + \E{\runTime_2}$.
    We consider the expectations of both phases separately.

    \textbf{Phase~$\bm{1}$.}
    Note that the time until the algorithm finds an optimum of~$\shiftedOneMax$ is an upper bound for~$\runTime_1$.
    Hence, the same is true for their expectations.
    By \Cref{cor:runTimeOnShiftedOneMax},
    we obtain $\E{\runTime_1} = \bigO{\lowerBounds[1]^{-1} n \log(n)}[\big]$.

    \textbf{Phase~$\bm{2}$.}
    Note that before the global optimum is found, no rate enters the archive due to being successful, since the global optimum is the only remaining strictly better solution.
    We bound the expected time of this phase in two different ways.
    The first way relies exclusively on the lower bound~$\lowerBounds[\jumpParameter]$ for sampling the global optimum, and the second way considers a course of events similar to stagnation detection.
    It assumes that rate~$1$ is active in iteration~$\runTime_1$, that is, at the beginning of phase~$2$.
    Calling the respective bounds~$B_1$ and~$B_2$, we then conclude that $\E{\runTime_2} \leq \min\{B_1, B_2\}$, as either bound is valid.


    \textbf{1. Relying only on the lower bound.}
    In each iteration, the probability for choosing rate~$\jumpParameter$ is at least~$\lowerBounds[\jumpParameter]$.
    Then, since the current individual has exactly~$\jumpParameter$ $0$s, the probability to mutate it into the global optimum is $\binom{n}{\jumpParameter}^{-1}$.
    Hence, in each iteration, the probability to create the global optimum is at least $\lowerBounds[\jumpParameter] \binom{n}{\jumpParameter}^{-1} \eqqcolon q$.
    Since each iteration has an independent chance to create the global optimum,~$\runTime_2$ follows a geometric distribution with a success probability of at least~$q$.
    Hence, $\E{\runTime_2} \leq q^{-1} = \lowerBounds[\jumpParameter]^{-1} \binom{n}{\jumpParameter}$.

    \textbf{2. Rate~$\bm{1}$ is active at the start.}
    Let~$\rateOneActive$ denote the event that $1 \in \preferredRates[\runTime_1]$, and assume that~$\rateOneActive$ occurs.
    Then $\globalBound[\runTime_1] = \timeToNext[1] / \freqs{\runTime_1}[1] \leq \timeToNext[1] / \lowerBounds[1]$.
    Thus, if no strict improvement is found within at most $\timeToNext[1] / \lowerBounds[1]$ iterations, the algorithm reverts to classic stagnation detection.
    Since the current individual is on the plateau, finding a strict improvement means to find the global optimum (and thus ending the phase).

    Let~$\runTime_{2,1}$ denote the first time in phase~$2$ such that~$\globalCounter[\runTime_{2, 1} + \runTime_1] \geq \timeToNext[1] / \lowerBounds[1]$, and let $\runTime_{2, 2} = \runTime_2 - \runTime_{2, 1}$.
    Since no mutation rate but~$\jumpParameter$ can lead to creating the global optimum, the global counter increases at most $\timeToNext[1] / \lowerBounds[1]$ times.
    Hence, $\runTime_{2, 1} \leq \timeToNext[1] / \lowerBounds[1]$, and thus $\E{\runTime_{2, 1}}[\rateOneActive][] \leq \timeToNext[1] / \lowerBounds[1]$.

    After~$T_{2, 1}$, if the algorithm did not find the optimum yet, it performs classic stagnation detection by maintaining a single active rate, based on~$\timeToNext$.
    Since all rates in $[\jumpParameter - 1]$ cannot lead to an improvement, a total of $\sum_{i \in [\jumpParameter - 1]} \binom{n}{i} \ln(\stagnationWhp)$ iterations is spent until rate~$\jumpParameter$ becomes active.
    For $\jumpParameter \leq \frac{n}{2}$, we obtain by \Cref{lem:sumOfBinomials}
    \begin{align*}
        \sum\nolimits_{i \in [\jumpParameter - 1]} \binom{n}{i} \ln(\stagnationWhp)
        &\leq \frac{n - (\jumpParameter - 2)}{n - (2\jumpParameter - 3)} \binom{n}{\jumpParameter - 1} \ln(\stagnationWhp)\\
        &= \frac{n - (\jumpParameter - 2)}{n - (2\jumpParameter - 3)} \frac{\jumpParameter}{n - \jumpParameter + 1} \binom{n}{\jumpParameter} \ln(\stagnationWhp)\\
        &= \bigO{\frac{\jumpParameter}{n - 2\jumpParameter + 3}} \binom{n}{\jumpParameter} \ln(\stagnationWhp) .
    \end{align*}
    For $\jumpParameter > \frac{n}{2}$, we obtain $\sum\nolimits_{i \in [\jumpParameter - 1]} \binom{n}{i} \ln(\stagnationWhp) \leq 2^n \ln(\stagnationWhp)$.

    Overall, for $\gamma \coloneqq \bigO{\frac{\jumpParameter}{n - 2\jumpParameter + 3}} \binom{n}{\jumpParameter} + 2^n \cdot \indicator{\jumpParameter > \frac{n}{2}}$, it takes at most $\gamma \ln(\stagnationWhp)$ iterations until rate~$\jumpParameter$ is active or the optimum is found.

    Assume that the optimum is not found until rate~$\jumpParameter$ is active.
    Let~$S$ denote the event that the algorithm samples the optimum within the next $\timeToNext[\jumpParameter] \ln(\stagnationWhp)$ iterations.
    The probability (unconditional on~$S$) to create the optimum in a single iteration is $(1 - \sum_{i \in [n] \smallsetminus \{\jumpParameter\}} \lowerBounds[i]) \cdot \binom{n}{\jumpParameter}^{-1} = L^{-1} \binom{n}{\jumpParameter}^{-1}$.
    The process of creating the optimum, conditional on~$S$, follows a truncated geometric distribution~\cite{RajabiW23StagnationDetection} and is dominated by a geometric distribution with success probability~$L^{-1} \binom{n}{\jumpParameter}^{-1}$, since the actual process stops after at most~$\timeToNext[\jumpParameter] \ln(\stagnationWhp)$ attempts.
    Hence, $\E{\runTime_{2, 2}}[\rateOneActive, S][] \leq \gamma \ln(\stagnationWhp) + L \binom{n}{\jumpParameter}$.

    In the case that~$S$ does not hold, we use the same arguments as when relying only on~$\lowerBounds[\jumpParameter]$ and obtain $\E{\runTime_{2, 2}}[\rateOneActive, \overline{S}][] \leq \gamma \ln(\stagnationWhp) + \lowerBounds[\jumpParameter]^{-1} \binom{n}{\jumpParameter}$.

    Next, we trivially bound $\Pr{\rateOneActive, S} \leq 1$ as well as $\Pr{\rateOneActive, \overline{S}}[][] \leq \Pr{\overline{S}}[\rateOneActive][]$.
    By the definition of~$S$, the estimates above of finding the global optimum in a single iteration, as well as that $\binom{n}{\jumpParameter} = \timeToNext[\jumpParameter]$ due to the standard-SD choice, $\Pr{\overline{S}}[\rateOneActive][] \leq \big(1 - L^{-1} \timeToNext[\jumpParameter]^{-1}\big)^{\timeToNext[\jumpParameter] \ln(\stagnationWhp)} \leq \stagnationWhp^{-1/L}$.
    Thus, we obtain
    \begin{align*}
        \E{\runTime_2}[\rateOneActive]
        &\leq \frac{\timeToNext[1]}{\lowerBounds[1]}
            + \gamma \ln(\stagnationWhp)
            + L \binom{n}{\jumpParameter}
            + \lowerBounds[\jumpParameter]^{-1} \binom{n}{\jumpParameter} \stagnationWhp^{-1/L} .
    \end{align*}

    We conclude by trivially bounding $\Pr{\rateOneActive} \leq 1$ and by adding the term $\timeToNext[1] / \lowerBounds[1] = n \lowerBounds[1]^{-1} \ln(R)$ also to our first bound for phase~$2$.
\end{proof}

\Cref{thm:runTimeOnJump} shows the intricate interplay of the parameters of the \flexEA.
For $\jumpParameter \geq \frac{n}{2}$, the value of the lower bound for rate~$\jumpParameter$ has a huge impact on the expected runtime, as it is better to choose rate~$\jumpParameter$ due to its lower bound than to wait for it to become the solely active rate.
The latter requires to wait in the order of at least the central binomial coefficient, which is exponential in~$n$.

In contrast, if $k < \frac{n}{2}$ is sufficiently small, the value~$L$ of all probability mass put on rate~$\jumpParameter$, except for the lower bounds, plays an important role, as it can speed up the process of choosing rate~$\jumpParameter$.
If the optimum is not found when rate~$\jumpParameter$ is active, the algorithm uses the lower bound of rate~$\jumpParameter$ again.
However, choosing the parameter~$\stagnationWhp$ sufficiently large makes it unlikely for this case to occur.

The following corollary summarizes how the expected runtime improves with the recommended parameter choices.
Up to a factor of $2 + \smallO{1}$, the result is asymptotically optimal if~$\jumpParameter$ is sufficiently far away from~$\frac{n}{2}$.
The factor of~$2$ is a result of the lower bounds, which take up a total probability mass of~$\frac{1}{2}$ in the recommended setting.
Hence, even if rate~$\jumpParameter$ is the only active rate, the probability of choosing it is only about~$\frac{1}{2}$.

\begin{corollary}
    \label{cor:runTimeOnJumpPowerLaw}
    Let $\powerLawExponent \in (1, 2)$ and $\varepsilon \in \R_{> 0}$ be constants, let $\stagnationWhp = n^{2(\beta + \varepsilon)}$, and let $\jumpParameter \in [2 .. n]$.
    Consider the \flexEA with the standard-SD choice~$\timeToNext$ with parameter~$\stagnationWhp$ and with heavy-tailed lower bounds~$\lowerBounds$.
    Furthermore, consider that the \flexEA optimizes $\jump_\jumpParameter$.
    Then the expected runtime is
    \begin{align*}
        &\big(1 + \smallO{1}\big)\min \Bigg\{
                \binom{n}{\jumpParameter} \cdot 2 \powerLawNormalization[\beta] k^{\beta},\\
                &\ \binom{n}{\jumpParameter} \cdot \left(2
                    + \bigO{\frac{\jumpParameter}{n - 2\jumpParameter + 3} \log(n) + n^{-\varepsilon}}
                \right)
                + 2^n \cdot \indicator{k \geq \frac{n}{2}} \ln(\stagnationWhp)
            \Bigg\} .
    \end{align*}
    Especially, for $k = \smallO{\frac{n}{\log(n)}}$, this simplifies to $\big(2 + \smallO{1})\binom{n}{\jumpParameter}$.
\end{corollary}

\begin{proof}
    The first bound follows from \Cref{thm:runTimeOnJump}.
    Due to the assumptions on~$\lowerBounds$, it holds that $\lowerBounds[1]^{-1} = \bigTheta{1}$, that $\lowerBounds[\jumpParameter] = \frac{1}{2 \powerLawNormalization[\beta]} \jumpParameter^{-\beta}$, and that $1 - \sum_{i \in [n] \smallsetminus \{\jumpParameter\}} \lowerBounds[i] \geq \frac{1}{2}$.
    Furthermore, due to the choice of~$\stagnationWhp$, it follows that $\bigO{n \lowerBounds[1]^{-1} \log(n \stagnationWhp)} = \bigO{n \log(n)}$.
    Since $\jumpParameter \geq 2$, the other summand is $\bigOmega{n^2}$, so the
    summand
    $\bigO{n \lowerBounds[1]^{-1} \log(n \stagnationWhp)}$
    is of lower order.
    Last, since $L^{-1} \geq \frac{1}{2}$, it follows that $\stagnationWhp^{-1/L} \leq \stagnationWhp^{-1/2} \leq n^{-(\beta + \varepsilon)}$.
    Substituting all of the values yields the desired result.

    For the case $k = \smallO{\frac{n}{\log(n)}}$, note that $\frac{\jumpParameter}{n - 2\jumpParameter + 3} \log(n) = \smallO{1}$.
    Hence, the second term in the minimum is $\big(2 + \smallO{1})\binom{n}{\jumpParameter}$.
    This term is smaller than the first term in the minimum, since $\powerLawNormalization[\beta] \geq 1$ as well as $\jumpParameter^\beta \geq 2$, concluding the proof.
\end{proof}

\section{Benefits of the Rate Archive on
Minimum Spanning Trees}
\label{sec:mst}
The prior works on \sdrlss (\ie, RLS with a classical stagnation detection mechanism)
in \cite{RajabiW23StagnationDetection} and on \sdrlsm, a variant with a radius memory,
in \cite{RajabiWittGECCO21} include analyses of the algorithms on the classical minimum spanning tree problem.
Interestingly, both stagnation detection
algorithms find optimal solutions to the problem in efficient polynomial expected time (see detailed expressions below),
which stands in contrast to the bound
$\bigO{m^2(\log n+\log w_{\max})}$, where $n$ is the number of vertices, $m$ the number of edges
and $w_{\max}$ is the largest edge weight, from the seminal work \cite{NeumannW07} analyzing the (1+1)~EA on the MST problem.
It is
well known \cite{ReichelSkutellaFOGA09} that the dependency of the bound on $w_{\max}$ can be removed and
that the runtime bound $\bigO{m^3\ln n}$ is obtained by studying the (1+1)~EA on a rank-equivalent
fitness function, however, even that bound is probably far from tight.
Before that, polynomial bounds were only known for randomized local search algorithms
with $1$- and $2$-bit flip neighborhoods, but not for globally searching algorithms.

In this section, we will
analyze the \flexEA on the MST problem and prove results matching the best classical stagnation detection algorithm.
The setting is as follows.
Given is an undirected, weighted graph $G=(V,E)$,
where $n=\card{V}$, $m=\card{E}$ and  the weight of edge $e_i$, where $i\in\{1,\dots,m\}$,
is a positive integer~$w_i$. Let $c(x)$ denote the number of connected components in the
subgraph described by the search point $x\in\{0,1\}^m$.
The fitness function $f\colon\{0,1\}^m\to\R$  considered in~\cite{NeumannW07}, to be minimized,
is defined by
\[
f(x)\coloneqq 
M^2 (c(x)-1) + M \left(\sum\nolimits_{i \in [m]} x_i - (n-1)\right) + \sum\nolimits_{i\in[m]} w_ix_i
\]
for an integer $M\ge n^2 w_{\max}$, where $w_{\max}$ denotes the largest edge weight.
Hence, $f$
returns the total weight of a given spanning tree and penalizes unconnected graphs as well
as graphs containing cycles so that such graphs are always inferior than spanning trees. The authors of~\cite{RajabiW23StagnationDetection}
showed that \sdrlss
 starting
with an arbitrary spanning tree finds an MST in $(1+o(1)) \bigl(m^2\ln m
+ (4m\ln m)\E{S}\bigr)$ fitness calls where $\E{S}$ is the expected number
of strict improvements. The variant \sdrlsm with radius memory analyzed in \cite{RajabiWittGECCO21} has an expected runtime of at most
\[
(1+o(1))\bigl((m^2/2)(1+\ln({\textstyle\sum_{i \in [m]}}r_i))\bigr)
 \le (1+o(1)) \bigl(m^2\ln m \bigr),
\]
where $r_i$ is the rank of the $i$th edge in the sequence sorted
by increasing edge weights. This holds since \sdrlsm for a sufficiently long time only flips two bits uniformly at random; then its stochastic search trajectory on the given MST instance is indistinguishable from an instance, where the $i$-largest weight from the original set $w_1,\dots,w_m$ is replaced by the rank
number~$r_i$ \cite{RaidlKJ06}.
Essentially, the bound
for \sdrlsm is better by a factor of~$2$ than for the classical  \rlsonetwo algorithm
choosing uniformly at random
a rate~$r\in\{1,2\}$ and flipping~$r$ bits uniformly. In fact,
\rlsonetwo wastes
every second step since
$1$-bit flips are never accepted on spanning trees.


We obtain Theorem~\ref{theo:mst} below. Choosing $\lowerBounds[1]=\smallO{1}$, it matches the bound from \cite{RajabiWittGECCO21}. While it does not improve the bound,
which seems to be the best obtainable with
the present methods, it shows the robustness of the \flexEA and features a simpler analysis than in \cite{RajabiWittGECCO21}.
Finally, it assumes a uniform starting point instead of an initialization with a spanning tree as in\cite{RajabiWittGECCO21}.

We note that the parameter
settings for the $\timeToNext[i]$ and particularly the $\lowerBounds[i]$ in Theorem~\ref{theo:mst} deviate from the
recommended settings discussed in \Cref{sec:flexEA:recommendation} and used earlier in this paper. In fact,
we take a gray-box perspective here and use a minor amount of problem-specific knowledge to  favor the
small rates~$1$ and~$2$. If the expected value
was to hold
conditioned on an event of probability $1-\smallO{1}$ only,
then the conditions on $\lowerBounds[i]$ for $i\ge 3$
could be relaxed.

\begin{theorem}
\label{theo:mst}
    Consider an instance to the MST problem, modeled
with the classical fitness function from \cite{NeumannW07}. Let \flexEA with the following
parameter settings run on this function:
\begin{itemize}
\item $\timeToNext[i] = m^i\ln (m^c)$ for a sufficiently large constant~$c>0$ and $i=1,2$
 (the other $\timeToNext[i]$ may be chosen arbitrarily),
        \item $\lowerBounds[1] = \smallOmega{1/\ln m}$
        and $\lowerBounds[1] \le 1/2$
        \item
 $\lowerBounds[2]  = \bigOmega{1/(m^2\ln m)}$,
 and
        $\lowerBounds[i] = \smallO{1/(m^5\ln^2 m)}[\big]$ for  $i=3..n$.
        \end{itemize}
Then with probability $1-\smallO{1}$, the runtime is at most
\begin{align*}
& ((1-\lowerBounds[1])^{-1}+\smallO{1})\bigl((m^2/2)(1+\ln({\textstyle\sum_{i \in [m]}}r_i))\bigr)\\
& \quad \le ((1-\lowerBounds[1])^{-1}+\smallO{1}) \bigl(m^2\ln m \bigr),
\end{align*}
where $r_i$ is the rank of the $i$th edge in the sequence sorted
by increasing edge weights. The expected runtime
is bounded in the same way.
\end{theorem}

The proof had to be omitted from this paper. Its main idea is to consider a phase of length
$\bigTheta{(1-\lowerBounds[1])^{-1}m^2 \ln m}$ after
finding the first spanning tree and to show that
within the phase, the archive only contains
rates~$1$ and~$2$ \whp Then the above-mentioned
stochastic equivalence to RLS with a 2-bit flip
neighborhood is applied.

\section{A Function Where the Archive
Stores Several Rates Simultaneously}
\label{sec:twoRates}
Most of the above  results  do not exploit
the full power of the archive~$\preferredRates$;
most of the time, it was sufficient that the archive
contained only one rate and larger archive sizes
were even detrimental for the analysis.
A major difference of the \flexEA
to classical stagnation detection is
that  a successful rate is kept for
the next iteration; however, to some extent this idea
was already present in stagnation detection with
radius memory \cite{RajabiWittGECCO21}, which
we will compare against.

The purpose of this section is to present an example where
\flexEA needs to
focus on two rates simultaneously. This can be beneficial on a multimodal function where fitness
gaps of different sizes occur frequently. Our example
resembles a \textsc{Hurdle} function \cite{NguyenSudholtGECCO18} and
features two different gap sizes,
so that the archive should
include the two corresponding
rates.
It seems possible to extend the
example
to a larger number
of gap sizes.

Assume the
numbers $s\coloneqq (3/4)n-\sqrt{n}$, $(3/4)n$,
$g\coloneqq \log n$ and
$\sqrt{n}/g$ are
integers. We define
\begin{equation*}
\TwoFrequencies(x) \coloneqq
\begin{cases}
    \oneMax(x) & \text{if $\card{x}< s$ or $\card{x}\ge (3/4)n$} \\&
    \text{or $\card{x} = s + ig$  for $i\in\{0,2,3,5,6,\dots\}$},\\
    -1 & \text{otherwise},
\end{cases}
\end{equation*}
where the values of~$i$ alternatingly
increase by~$2$ and~$1$.

The hurdles of width~$g$, which
are located in the interval
$[s .. 3n/4]$ for
the number of one-bits,
are overcome
at the corresponding rate
with probability at least
\begin{align*}
\frac{\binom{n/4}{g}}{\binom{n}{g}}
\ge \left(\frac{n/4-g}{n}\right)^{g} =
\left(1-\frac{\bigO{\log n}}{n}\right)^{\log n}4^{-\log n} =  \frac{(1-\smallO{1})}{n^{2}}
\end{align*}
(using the inequalities $\tfrac{(n-k)^k}{k!} \le \binom{n}{k} \le \tfrac{n^k}{k!}$) and at most
\begin{align}
    \frac{\binom{n/4+\sqrt{n}}{g}}{\binom{n}{g}} 
   &
\le 
\left(\tfrac{n/4+\sqrt{n}}{n-g}\right)^g =
 2^{\log(1/4+\bigO{1/\sqrt{n}})(\log n)}\le \tfrac{(1+\smallO{1})}{n^{2}}
 \label{eq:upper-hurdle}
\end{align}
and accordingly those of width $2g$ with the probability bounds
$(1-\smallO{1})/n^4$ and
$(1+\smallO{1})/n^4$, respectively.

We now prove that the \flexEA optimizes \TwoFrequencies efficiently, while classical stagnation detection algorithms with and without radius memory need superpolynomial time. Moreover, the
\fastea
from \cite{DoerrLMNGECCO17} is
by a at least a polylogarithmic factor
of $\bigOmega{(\log n)^{\beta-1/2}}[\big]=\smallOmega{\sqrt{\log n}}[\big]$ slower
than our approach.

\begin{theorem}
    \label{theo:twofreqcs}
    The expected runtime of the \flexEA with
    the recommended parameter set on \TwoFrequencies is
    bounded by $\bigO{n^{4.5}}$.
    For the stagnation detection algorithms of the type
    SD-RLS
 in \cite{RajabiWittGECCO21, RajabiW23StagnationDetection} with typical
parameter choice $R=n^d$ for a constant~$d>0$, the expected
runtime is $n^{\bigOmega{\log n}}$. Moreover,
the expected runtime
of the \fastea is by
a factor $\bigOmega{(\log n)^{\beta-1/2}}[\big]$ larger
than for the \flexEA.
\end{theorem}

In the proof,  we
shall need the following useful helper result. It relates to the fact that flipping exactly $g$ zero-bits and no one-bits
may not be optimal to cross a hurdle of size~$g$, and it bounds the increase in probability when flipping both
zero- and one-bits.

\begin{lemma}
\label{lem:nocoupling}
    Let $x\in\{0,1\}^n$ be a search point
    with $a\ge 0$ one-bits and let $p(n,a,d,i)$, where $d,i\ge 0$, be
    the probability of mutating $x$ into a  point containing
    exactly $a+d$ one-bits when flipping $d+2i$ bits uniformly at
    random. If $d+2i\le n^{1/3}$, then $p(n,a,d,i)/p(n,a,d,0)\le (1+\smallO{1}) (\eulerE(2+d/i)(1-a/n))^i \le (1+\smallO{1})\eulerE^{d+2i} (1-a/n)^i $.
\end{lemma}

The main idea for the upper bound for the \flexEA is as follows: the expected time to cross a hurdle is $\bigO{n^4}$ and there are $\bigTheta{\sqrt{n}/\log n}$ hurdles. The probability of choosing a beneficial rate of~$g$ or~$2g$ is
$\bigOmega{1/\log n}$ since the cardinality of the archive is bounded by $\bigO{\log n}$ with high probability. The lower bound for the \fastea stems from the fact that rates of $\bigOmega{\log n}$ must be chosen with high probability to cross a hurdle, which has probability $\bigO{(\log n)^{-\beta}}[]$ in the heavy-tailed mutation. Another factor of $\bigTheta{\sqrt{\log n}}[]$ is lost since the \fastea uses standard-bit mutation instead of a flipping a deterministic number of bits.

We conjecture that the runtime
of the \fsdrls from
\cite{DoerrR23} on
\TwoFrequencies is asymptotically at least as large as the one of the  \fastea. Essentially, rates in the order $\bigOmega{\log n}$ must be
used
to obtain a sufficient probability of
crossing a hurdle. Then in the same
way as described above, the selection
probability $\bigO{(\log n)^{-\beta}}[]$
in the heavy-tailed component impacts
the runtime bound.

\section{Conclusion}
\label{sec:conclusion}

We introduced the \flexEA~-- a flexible evolutionary algorithm that maintains an archive of previously successful mutation rates.
The algorithm is highly configurable and thus allows to exploit problem-specific knowledge.
However, we also recommended a default configuration, based on heavy-tailed distributions and stagnation detection, which results in an essentially parameter-less algorithm.
Using this recommended configuration, we proved that the \flexEA achieves standard runtimes on \oneMaxName and on \leadingOnesName, and that it achieves a near best-possible runtime on \jump for unbiased mutation-only algorithms.
Furthermore, for a hurdle-like problem that features repetitive sub-structures, this configuration outperforms the \fastea (which uses heavy-tailed mutations) and stagnation detection.
When provided with some mild problem-specific knowledge, the \flexEA 
matches the best known runtime result for the minimum spanning tree problem.

Although the performance of the \flexEA is very good on the problems we considered, the algorithm still has some shortcomings that can be improved.
For one, on shifted jump benchmarks~\cite{Witt23MajorityVote}, the combination of heavy-tailed mutation and stagnation detection by \citet{DoerrR23} should prove superior, as the \flexEA basically defaults to classic stagnation detection.
However, this shortcoming can be overcome by adjusting the mutation of the \flexEA in a similar way as done by \citet{DoerrR23} for their algorithm.
Another potential shortcoming is that the archive of the \flexEA can become very crowded when many different mutation rates prove useful.
If the mutation rates are removed in an incorrect order, this can result in long waiting times.
It is unclear under what circumstances this effect actually occurs.

Overall, the \flexEA proposes an arguably new paradigm for evolutionary algorithms that performs very well and can be studied in multiple domains that exceed the content of this paper.

\clearpage
\bibliographystyle{ACM-Reference-Format.bst}
\balance\bibliography{references.bib}

\clearpage
\appendix
\onecolumn
\section{Appendix}

This appendix contains material that does not fit the main paper, due to space restrictions.
It is meant to be read at the reviewer's discretion only.
It contains mathematical tools that we use in some of our proofs and proofs that had to be omitted from the main paper.

\section*{Unimodal Functions}

The fitness level method was introduced by \citet{Wegener01} for the analysis of evolutionary algorithms on a given fitness function~$\fitness$.
It requires a partition $(L_i)_{i \in [m]}$ of the search space $\{0, 1\}^n$ (with $m, n \in \N_{\geq 1}$) such that the fitness increases in each level.
That is, for all $i \in [m - 1]$ and all $\individual \in L_i$ and $\individualOther \in L_{i + 1}$, it holds that $\fitness[\individual] < \fitness[\individualOther]$.
We call such a partition a \emph{fitness level partition} and its elements \emph{fitness levels.}

The following theorem by \citet{Sudholt13FitnessLevels} is used in the proof of \Cref{thm:runTimeOnUnimodal}.

\begin{theorem}[{\cite[Theorem~$2$]{Sudholt13FitnessLevels}}]
    \label{thm:fitnessLevelMethod}
    Let~$\fitness$ be a pseudo-Boolean function, let $m \in \N_{\geq 1}$, and let $(L_i)_{i \in [m]}$ be a fitness level partition of~$\fitness$.
    Let~$A$ be an elitist search algorithm, and, for all $i \in [m - 1]$, let~$s_i$ denote a lower bound on the probability of~$A$ creating a new search point in $\bigcup_{j \in [i + 1 .. m]} L_j$, provided that the best individual created so far by~$A$ is in~$L_i$.
    Then the expected optimization time of~$A$ on~$\fitness$ (without the initialization) is at most $\sum\nolimits_{i \in [m - 1]} \frac{1}{s_i}$.
\end{theorem}

\section*{Jump Functions}

The following lemma by \citet{RajabiWittGECCO21} is used in the proof of \Cref{thm:runTimeOnJump}.

\begin{lemma}[{\cite[Lemma~$1$~(a)]{RajabiWittGECCO21}}]
    \label{lem:sumOfBinomials}
    Let $n, \jumpParameter \in \N$ with $\jumpParameter \leq \frac{n}{2}$.
    Then $\sum_{i \in [\jumpParameter]} \binom{n}{i} \leq \frac{n - (\jumpParameter - 1)}{n - (2\jumpParameter - 1)} \binom{n}{\jumpParameter}$.
\end{lemma}

\section*{Benefits of the Rate Archive on Minimum Spanning Trees (Section~\ref{sec:mst})}

Drift analysis is the most prominent mathematical tool for analyzing evolutionary algorithms.
It was introduced by \citet{HeY01DriftTheory}.
We use the following theorem by \citet{DoerrJohannsenWinzenALGO12} in the proof of \Cref{theo:mst}, presented in  a version close to the one by \citet{KotzingK19DriftTheory}. The tail bound
can be found in
\cite{DoerrGoldbergAlgo13} and \cite{LenglerDriftBookChapter}.

\begin{theorem}[{\cite[Theorem~$3$]{DoerrJohannsenWinzenALGO12}}]
    \label{thm:multiplicativeDrift}
    Let $(X_t)_{t \in \N}$ be a Markov chain over~$\R_{\geq 0}$, and let $T = \inf \{t \in \N \mid X_t < 1\}$.
    Suppose that there is a $\delta > 0$ such that it holds that $\E{X_t - X_{t + 1}}[X_t] \cdot \indicator{t < T} \geq \delta X_t \cdot \indicator{t < T}$.
    Then $\E{T} \leq (1 + \E{\ln(X_0)}) / \delta$. Moreover, $\Pr{T\ge
    (\E{\ln(X_0)}+r)} \le \eulerE^{-r}$.
\end{theorem}

We now give the full proof of Theorem~\ref{theo:mst} that was omitted from the main paper.

\begin{proofof}{Theorem~\ref{theo:mst}}
    Similarly to the classical analysis of the (1+1)~EA on the problem \cite{NeumannW07}, we divide the run into three phases. The first two phases are considered jointly. They can find improvements in the $1$-bit flip
    neighborhood by including edges (\ie, flipping
    zero-bits) and reducing the number of connected components or by reducing the number of edges from a fully connected
    subgraph by flipping one-bits. The expected number of $1$-bit flips until the first phase is finished is bounded by $\bigO{m\log m}$ in expectation, as shown in
    \cite{NeumannW07} using fitness-level arguments. Using tail bounds for fitness levels \cite{WittIPL14}, the number is $c'm\ln m$ for some constant~$c'>0$
    with probability $1-\smallO{1}$.
    As long as fitness improvements by 1-bit flips are still available, the expected time
    until choosing rate~$1$ is at most~$1/\lowerBounds[1]=\smallO{\log m}$, so both
    in expectation and with high probability after $\smallO{m\log^2 m}[]$
    steps
    a local optimum with respect to the $1$-bit neighborhood
    has been reached and the current search point constitutes a (not necessarily minimum) spanning tree. If the constant $c$ from $\timeToNext[1]$ is chosen
    sufficiently large, then the failure counter
    $c_1$ does not exceed the threshold $\timeToNext[1]$ during the
    considered number of steps at rate~$1$
    with probability
    at least~$1-\smallO{1/m^2}$. Moreover, by a union bound, the probability
    that mutation rates higher than $2$ are
    ever chosen in the considered period is \[
    \bigO{m\log^2 m} \sum_{\ell=3}^m \lowerBounds[i] = \smallO{1/m^3},
    \]
    which we assume not to happen. Then no rates
    larger than~$2$ are in the archive $\preferredRates[t]$ at the end of the period.
    If rate~$2$ is not yet present at that time,
    then the global counter $\globalBound[t]$ exceeds
    its threshold after $\bigO{m\ln^2 m}[]$ steps~--
    and rate~$2$ enters the archive after another
    $\bigO{\timeToNext[1]/\lowerBounds[1]} = \bigO{m\ln^2 m}[]$ expected number of steps at rate~$1$,
    all of which are unsuccessful. Again, the probability of rates larger than~$2$ entering
    the archive during that time
    is $\smallO{1/m^3}$ by the same
    arguments as above.

    After the current search point constitutes a spanning tree for the first time, the final
    phase
    begins. Here only steps flipping an even number of bits
    can be accepted since disconnected graphs and graphs having more than $n-1$ edges lead to inferior fitness values.
     We now concentrate on $2$-bit flips and aim at using multiplicative drift analysis using
$g(x)=\sum_{i=1}^m x_ir_i$ as potential function, where $r_i$ is defined in the statement of the theorem.
As shown in \cite[Proof of Lemma~3]{RaidlKJ06},
RLS flipping one or two bits uniformly at random behaves stochastically identically on the original
fitness function~$f$ and the function~$g$, assuming
a spanning tree as starting point. Hence,
our plan is to choose a sufficiently large constant~$c'>0$ and to conduct the drift analysis conditioned
on that only rates~$1$ and~$2$ are chosen in a phase of predefined length~$c'(1-\lowerBounds[1])^{-1} m^2\ln m = \bigO{m^2 \ln^2 m}[]$, using $\lowerBounds[1]\le 1/2$. Note that in the phase, the
steps of rate~$1$ cannot be successful.
The probability of only choosing rates at most~$2$ throughout the phase is at least
\[
1-\bigO{m^2\ln^2 m} \sum_{\ell=3}^m \lowerBounds[i] =
1- \bigO{1/m^2} .
\]
In the following, we assume the phase to fulfill this property. We choose $c$ small enough to bound
the probability of the error that $\counts{\cdot}[2]$
exceeds its threshold $\timeToNext[2]$ by $\bigO{m^{-2}}$
and analyze the error case at the end of the proof.

Let $X^{(t)}\coloneqq g(x_t)-g(\xopt)$ for the current
search point~$x_t$ and an optimal search point $\xopt$.
Since \dsdRLS under the assumption of choosing only rates~$1$ and~$2$ behaves stochastically
the same on the original fitness function~$f$ and the potential
function~$g$, we obtain that $\E{X^{(t)} - X^{(t+1)}}[X^{(t)}][]
\ge X^{(t)}/\binom{m}{2}\ge 2X^{(t)}/m^2$
since the $g$-value can be
decreased by altogether
$g(x_t)-g(\xopt)$ via a sequence of at most
$\binom{m}{2}$ disjoint two-bit flips; see also the proof of
Theorem~15 in \cite{DoerrJohannsenWinzenALGO12} for the underlying
combinatorial argument. Let $T$ denote the number of steps at
strength~$2$
until $g$ is minimized, assuming no larger strength to occur.
Using the multiplicative drift theorem (\Cref{thm:multiplicativeDrift}) 
we have $\E{T}\le (m^2/2)(1+\ln(r_1+\dots+r_m)) \le
(m^2/2)(1+\ln(m^2))$ and by
the tail bounds for multiplicative drift (\eg, \cite{LenglerDriftBookChapter})
it holds that
$\Pr{T> (m^2/2)(\ln(m^2) + \ln(m^3))} \le \eulerE[-\ln(m^3)] = 1/m^3$. We choose a
phase length of $4(1-\lowerBounds[1])^{-1}m^2\ln m$, which includes at least~$m^2\ln m$ steps choosing
rate~$2$ with probability $1-\eulerE[-\bigOmega{m^2\ln m}]$
according to Chernoff bounds \cite[Theorem~1.10.5]{Doerr2020BookChapter}. Assuming this to happen,
we have bounded the probability of not finding the optimum in Phase~$2$ of that length is
$\smallO{m^{-2}}$.
The total probability of any error as described
above is still $\smallO{m^{-2}}$.

To bound the total expected optimization time,
we note
that we already have proved
an upper bound $\smallO{m^2\log m}$
on the expected time to reach a spanning tree
conditional on that no error occurs. If an error occurs, we wait for an accepted one-bit flip,
which happens after expected time $\bigO{m\log m}$
and repeat the argumentation. The expected number
of repetitions is $\smallO{1}$.

We now bound the expected time for the final phase. Hence, we assume a spanning tree as current search point.
If an error occurs, we have no real control over
the archive $\preferredRates[t]$. However, after an
expected number of at most $(1/\lowerBounds[2])m^2 = \bigO{m^4\ln m}$
steps, an accepted two-bit mutation adds rate~$2$
to the archive and progress can be made by flipping
two bits as above, with the caveat that the
stochastically equivalent behavior as on the
function~$g$ cannot be assumed any longer. Hence,
 we use a weaker bound
on the expected optimization time. More precisely,
 we replace the fitness function$f$ by the
 fitness function~$f'$ under which the $2^m$
 search points have the same ranking as on~$f$,
 but the maximum
 weight is $\bigO{m^m}$ \cite{ReichelSkutellaFOGA09}.
 Since the \flexEA is ranking based, it has the same
 stochastic behavior on~$f$ and~$f'$.
 On the process~$X_t'\coloneqq f'(x_t)$, using the classical
 analysis by \cite{NeumannW07},
 there is a multiplicative
 drift with parameter $\delta=\bigOmega{q/m^2}$, where
 $q$ is a lower bound on a step flipping two bits, while
 steps flipping more bits do not harm the analysis.
 Clearly, $q=\bigOmega{1/m}$ since no more than~$m$
 rates can be in the archive.
 Hence,
 we obtain an expected runtime of $\bigO{m\cdot m^2 \ln(m^m)} = \bigO{m^4 \ln m}$ in the failure case.
By the law of total probability, using
the failure probability $\smallO{m^{-2}}$, the
total expected runtime is still
$(1-\lowerBounds[1])^{-1}+\smallO{1})(m^2/2)(1+\ln(r_1+\dots+r_m))$.
\end{proofof}

\section*{A Function Where the Archive
Stores Several Rates Simultaneously (Section~\ref{sec:twoRates})}

We now give the full proofs of \Cref{lem:nocoupling} and \Cref{theo:twofreqcs} that were omitted from the main paper.

\newcommand{\poch}[2]{#1^{\underline{#2}}}

\begin{proofof}{\Cref{lem:nocoupling}}
To reach $a+d$ one-bits by flipping $d+2i$ bits, exactly $d+i$ zero-bits and $i$ one-bits must be flipped. Hence, we assume that $d+i\le n-a$ and $i\le a$. By elementary combinatorics and expanding the binomial coefficients, we have
\[
\frac{p(n,a,d,i)}{p(n,a,d,0)} =
\frac{\frac{\binom{n-a}{d+i}\binom{a}{i}}{\binom{n}{d+2i}}}{\frac{\binom{n-a}{d}}{\binom{n}{d}}}
= \frac{(d+2i)!}{i!(d+i)!} \cdot
\frac{\poch{(n-a)}{d+i}\cdot \poch{a}{i}}{\poch{n}{d+2i}} \cdot \frac{\poch{n}{d}}{\poch{(n-a)}{d}}
=
\frac{(d+2i)!}{i!(d+i)!} \cdot \frac{\poch{(n-a-d)}{i}\cdot \poch{a}{i}}{\poch{(n-d)}{2i}}
,
\]
where $\poch{b}{c} = b(b-1)\cdot \dots \cdot (b-c+1)$ is the falling factorial (sometimes
called \emph{Pochhammer} symbol).

We estimate the two fractions in the last expression separately and start with the second one. By our assumption that $d+2i\le n^{1/3}$, we have $\poch{(n-d)}{2i} \ge (n-d-2i)^{2i} \ge (n-n^{1/3})^{2i} =
(1-\smallO{1}) n^{2i}$ and therefore
\[
\frac{\poch{(n-a-d)}{i}\cdot \poch{a}{i}}{\poch{(n-d)}{2i}}
\le \frac{(n-a)^i n^i}{(1-\smallO{1}) n^{2i}}
\le (1+\smallO{1})\left(\frac{n-a}{n}\right)^{i}.
\]
For the first fraction, we have
\[
\frac{(d+2i)!}{i!(d+i)!} = \frac{\poch{(d+2i)}{i}}{i!} \le
\frac{(d+2i)^i}{(i/\eulerE)^i} = \left(\left(2+\frac{d}{i}\right)\cdot \eulerE\right)^i \le  \left(\eulerE^{1+d/i}\right)^i \eulerE^{i} = \eulerE^{d+2i},
\]
where the first inequality used the well-known estimate $i!\ge (i/\eulerE)^i$ following from Stirling's formula and the second inequality used
that $1+x\le \eulerE^x$ for all $x\in\R$.

Multiplying the two estimates gives the claim of the lemma.
\end{proofof}


\begin{proofof}{\Cref{theo:twofreqcs}}
We start with the upper bound for  the
\flexEA.

\paragraph{Reaching the first hurdle} If the initial search
point has less than $s\coloneqq (7/8)n-\sqrt{n}$
one-bits, then by \Cref{cor:runTimeOnShiftedOneMax}, the level of at least~$s$ one-bits is reached
within expected time $\bigO{n\log n/\lowerBounds[1]} = \bigO{n\log n}$ and rate~$1$ is in the archive with probability $1-\bigO{1/n^2}$ by choosing
the constant in
$\timeToNext[1]$ appropriately.
Moreover, by applying Markov's
inequality and repeating
independent phases, the
time to reach level at least~$s$ is  bounded by
$\bigO{n^4}$ with probability
$1-2^{-\smallOmega{n}}$.

We now consider the first time where
the current search point
has~$s$ or more one-bits and assume pessimistically exactly~$s$ one-bits.
The smallest feasible rate at this point is~$g$ and the probability of a success, \ie, increasing the number of one-bits by~$g$, at this rate is
$(1+\smallO{1})n^{-2}$
according to \eqref{eq:upper-hurdle}. Using
\Cref{lem:nocoupling},
 the probability of a success at rate~$g+2i$, where $i>0$, is bounded from above
by
\[
(1+\smallO{1})n^{-2}
\left(\left(2+\frac{g}{i}\right)\cdot \eulerE \cdot \left(\frac{1}{8} + \frac{1}{\sqrt{n}}\right)\right)^i
\le
(1+\smallO{1})n^{-2}
\left(\left(2+\frac{g}{i}\right)\cdot \eulerE \cdot \frac{1}{7.9} \right)^i
\]
using $a\ge s = 7n/8-\sqrt{n}$ and assuming $n$ large enough. The bracket attains its maximum at $i\approx 0.2868g$ (which can be verified by a standard computer algebra system) and is bounded from above by $1.89^{0.2868g} \le 2^{g/3} = n^{1/6}$  then. Altogether, the success probability for a hurdle of size~$g$ even at optimal rate is no more than $(1+\smallO{1}) n^{-11/6}$.

Therefore with probability $1-\smallO{1}$ by a union bound over $\globalBound =
\bigO{n\log n/\lowerBounds[1]}$ steps,
the global counter~$\globalCounter$ (not to be confused with the gap size)
will reach its threshold before
a hurdle is jumped over and the archive will
be reset to
contain rate~$1$ only. We assume
all this to happen.

\paragraph{Bounding the archive size.}
When the current search point has~$s$ one-bits, then successes
are only possible at rates~$2g$
or larger. The aim is to bound
the size of the archive while
the hurdles of width~$g$ and~$2g$ are overcome. We first rule out
that extremely large rates
yield a success and thereby enter the archive. Applying
a mutation flipping $d$ bits uniformly at random to a search point with $a\ge s$ one-bits, the
number of flipping
one-bits follows a hypergeometric distribution and the expected increase in one-bits equals
$ad/n - (n-a)d/n < 0$. Applying a Chernoff bound for the hypergeometric distribution (Theorem~1.10.25 in \cite{Doerr2020BookChapter}),
the probability
of increasing the number of one-bits by any positive amount when flipping a certain number~$d\ge n^{1/3}$ of bits in one step
is  bounded from above by
$
2^{-\bigOmega{n^{1/3}}[]}
$. By a union bound
over all $d\ge n^{1/3}$, the probability of increasing the number of one-bits using any rate at least~$n^{1/3}$ is
still $
2^{-\bigOmega{n^{1/3}}[]}
$.
 Hence,
we work under the assumption that rates larger than~$n^{1/3}$
are never successful in polynomial time; the error
probability is still $2^{-\bigOmega{n^{1/3}}[]}$ by
a union bound.

Clearly, the minimum successful rate is~$g$ while crossing the hurdles in the interval $[s .. (7/8)n]$ of one-bits. We now estimate the success
probabilities of rates
in $[g .. n^{1/3}]$.
By flipping $2g+i$ zero-bits and $i$ one-bits at rate~$2g+2i$, also larger
rates than~$2g$ can overcome hurdles
of size~$2g$. Using
\Cref{lem:nocoupling} using $d=2g$,
$i=c'\log n$ for
a sufficiently large
constant~$c'$, $a\ge s$ and our assumption
of rates at most~$n^{1/3}$,
we have that the probability of increasing the number of one-bits by $d$ at rate $d+2i$
is no larger than
a factor of
\[
(1+\smallO{1})
\eulerE[\log n] \eulerE[2i]
\left(\frac{1}{8}+\frac{1}{\sqrt{n}}\right)^{i} \le
(1+\smallO{1})
 \eulerE[\log n] 0.93^i
\]
larger than the probability
of doing this at rate $d$.
Hence,
for a sufficiently large constant~$c'>0$, rates larger
than~$2g+2c'\log n$ only
have a success probability of
$\bigO{n^{-5}}$. By a union bound,
the size of the archive is
bounded by $\bigO{\log n}$
within a period of $\bigTheta{n^{4.5}}$ steps with
probability $1-\bigO{n^{-1/2}}[]$.
We assume that the error
case of a larger archive does not happen either.

\paragraph{Crossing the hurdles
and reaching the optimum.}
In the following, under the
above assumptions,
we first show that
a phase of $c''n^{4.5}$  steps is
sufficient to cross all hurdles
by reaching at least~$7n/8$ one-bits
with high probability. No matter
whether there is a current
hurdle of width~$g$
or~$2g$ one-bits to the next
improvement, the
rate having the highest success probability (which is at least $(1-\smallO{1}) n^{-2}$ and $(1-\smallO{1}) n^{-4}$, respectively, but possibly larger
since it may be beneficial to flip more bits than the gap size to
overcome the gap), enters the archive after
$\bigO{\lowerBounds[2g] n^4}$
steps since it is sufficient to select the best rate and have a success (using that $\lowerBounds[g]\le \lowerBounds[2g]$). By Chernoff
bounds for geometrically distributed
random variables \cite[1.10.32]{Doerr2020BookChapter}, the time is
$\bigO{\lowerBounds[2g] n^4}$
with overwhelming probability.
Once both the optimal rate for gap~$g$ and for gap~$2g$
are in the archive, the probability
of choosing the rate fitting the
next hurdle is $\bigOmega{1/\log n}$
by our assumption on the archive
size and the time for an improvement
is $\bigO{n^4\log n}$ with
overwhelming probability and in
expectation (using that gap~$2g$
is the worst case). Since
at most~$\sqrt{n}/g = \bigO{\sqrt{n}/\log n}$ improvements
jumping hurdles are sufficient to
reach level $7n/8$ or larger
and the remaining part of the optimization is completed within
expected time $\bigO{(n\log n)/\lowerBounds[1]}$ in the
same way as described above, the
optimization time is $\bigO{n^{4.5}}$
in expectation and also with
high probability by Chernoff bounds for geometrically distributed random variables. In the error case of
a larger archive, we argue with the maximum archive size~$n$ and have a success probability of $\bigOmega{1/n^5}$. The expected time in the error case is therefore $\bigO{n^{5.5}/\log n}$. Finally, the
probability of exceeding a counter
is superpolynomially small since
$\timeToNext[g] \ge \binom{n}{\log n}$; if
this happens nevertheless, then we wait
for the optimum rates for gap sizes~$g$ and~$2g$
to enter the archive as above. By the law of total probability, this contributes a superpolynomially small extra term to the expected runtime.
Altogether,
the unconditional expected
optimization time of \flexEA
is $\bigO{n^{4.5}}$.

\paragraph{Lower Bounds for
stagnation detection and \fastea}
We are left with the bounds
for the other algorithms.
With overwhelming probability
$1-2^{-\bigOmega{n}}$, their first search point, which is
drawn uniformly, is before the first hurdle, \ie,
has at most $s$ one-bits. Hence, in the algorithms
\sdrlss and \sdrlsm, the rate must increase
from its initial value~$1$ to at
least~$g$ as
least once. With the
standard parameter choices, this
takes $\bigOmega{\sum_{i=1}^{g-1}\binom{n}{i}\ln R}[\big]=n^{\bigOmega{\log n}}$  steps.

The \fastea, under the assumption on the first
point, must increase the number of one-bits
 by a total of at least $\sqrt{n}$ to
cross all hurdles. We show that the algorithm does not cross
more than $h^*$ hurdles at once in a period of $\bigO{n^5}$ steps
with high probability, where $h^*$ is a sufficiently large constant. To cross
$h^*$ hurdles in one mutation, it is necessary that at least $cn$ zero-bits flip at once for
some constant~$c>0$.
By the results in \cite{DoerrLMNGECCO17}, for any $k\in [n/2]$, the probability of flipping
exactly $k$ bits is maximized at mutation probability $k/n$ (recall that the \fastea uses standard bit mutation).
Hence, the probability
of flipping $k\ge cn$ zero-bits and thereby crossing at least $h^*$ hurdles at once is at most
\[
\binom{n-s}{k} \left(\frac{k}{n}\right)^{k} \le
\left(\frac{(n-s)e}{k}\right)^{k} \left(\frac{k}{n}\right)^{k} \le \left(\frac{1}{8}+\frac{1}{\sqrt{n}}\right)^{k}
\eulerE[k] \le \left(\frac{\eulerE}{8}+\frac{\eulerE}{\sqrt{n}}\right)^{cg}=  \bigO{n^{-6}}[\big]
\]
if $h^*$, and thereby~$c$, are large enough.
By a union bound over all $k\ge cn$, the probability of crossing at least~$h^*$ hurdles at once is still
$\bigO{n^{-5}}$.
In the following, we assume that this does not happen
in $\bigO{n^{4.6}}$ steps.


Under our assumptions, to reach the optimum,  at least $\sqrt{n}/(cg)$
increases of the number of one-bits
by at least $cg$  in search
points at level~$s$ or larger are necessary, where
$c$ may be a large constant. We pessimistically
(for the perspective of a lower bound) assume that the
heavy-tailed mutation operator always flips a number of
bits $2g+2i$ that maximizes the probability of overcoming
a hurdle of width~$2g$ (and in the same way for the
hurdles of width~$g$).
As mentioned above, the
\fastea
flips $2g+2i$ bits with the
highest probability at mutation probability
$(2g+2i)/n$.
Hence, the probability of crossing
a hurdle of width~$2g$ is at most
\begin{align*}
& \binom{n-s}{2g+i} \binom{s}{i}
\left(\frac{2g+2i}{n}\right)^{2g+2i}
\left(1-\frac{2g+2i}{n}\right)^{n-2g-2i}
= p(n,s,2g,i) \cdot \left(\frac{2g+2i}{n}\right)^{2g+2i}\left(1-\frac{2g+2i}{n}\right)^{n-2g-2i}  \binom{n}{2g+2i}.
%
\end{align*}
Inspecting the product following $p(n,s,2g,i)$, which is the ratio of success probabilities
of the \fastea and
 of the \flexEA,
we have
\[
\left(\frac{2g+2i}{n}\right)^{2g+2i}\left(1-\frac{2g+2i}{n}\right)^{n-2g-2i}  \binom{n}{2g+2i} \le
\left(\frac{2g+2i}{n}\right)^{2g+2i} {\eulerE[-2g-2i+\smallO{1}]}  \frac{n^{2g+2i}}{(2g+2i)!},
\]
where the $\smallO{1}$ comes
from the assumption of $2g+2i\le n^{1/3}$.
Using Stirling's formula, we have $(2g+2i)! \ge \frac{1}{\sqrt{2\pi}} \sqrt{2g+2i}\cdot \eulerE^{-2g-2i} (2g+2i)^{2g+2i}$. Plugging this
into the previous formula, the bound on the ratio of success probabilities simplifies to
\[
(1+\smallO{1}) \sqrt{2\pi} \frac{1}{ \sqrt{2g+2i}},
\]
so the probability of the \fastea crossing a hurdle is
by a factor of $\bigOmega{\sqrt{2g+2i}} = \bigOmega{\sqrt{\log n}}$ smaller than when the
\flexEA has chosen the optimum radius.

Finally, we analyze the probability
of the \fastea choosing a successful rate
in the heavy-tailed mutation operator.
If its rate is at most $c'\log n$ for
a constant~$c'>0$, the
probability of crossing a hurdle of width~$g$ is at most
\[
\binom{n-s}{g}\left(\frac{c\ln n}{n}\right)^g \le \left(\frac{c'\ln n}{\ln n}\right)^{\ln n} = \bigO{1/n^5}
\]
choosing~$c'$ small enough. Hence, by
a union bound, no hurdle is crossed
at rate less than $c'\log n$ in a period of
$\bigOmega{n^{4.5}\sqrt{\ln n}}[\big]$ steps,
which we assume to happen.
On the other hand, the probability of the heavy-tailed mutation
choosing rate at least $c'\ln n$ is
at most $\bigO{{(1/\ln n)^{\beta}}}[\big]$.
This probability is by a factor $\bigO{{(1/\ln n)^{\beta-1}}}[\big]$ smaller than $1/\log n$, the
asymptotic worst-case probability
of the \flexEA choosing the best radius
from an archive of $\bigO{\log n}$ bits. Since both algorithms have to overcome
$\bigTheta{\sqrt{n}/\log n}$ hurdles with high probability and the success
probability of the \flexEA for crossing a hurdle is bigger by a factor
of $\bigOmega{\sqrt{\log n}}[\big]$,
we obtained the claimed asymptotic speed-up
of $\bigOmega{(\ln n)^{\beta-1/2}}[\big]$ for the \flexEA compared to the \fastea.
\end{proofof}

\end{document}